\documentclass[twoside]{article}
\usepackage[accepted]{aistats2020}
\usepackage{subfigure}
\usepackage{amsthm}
\usepackage{amssymb}
\usepackage{amsmath}
\usepackage{graphicx}
\usepackage{xcolor} % Remove when submit!!
\newcommand{\eqDef}{\triangleq}
\newcommand{\E}{\mathbb{E}}

\theoremstyle{definition}

\newtheorem{conjecture}{Conjecture}[section]
\newtheorem{theorem}{Theorem}[section]

\newtheorem{definition}[theorem]{Definition}
\newtheorem{lemma}[theorem]{Lemma}
\newtheorem{corollary}[theorem]{Corollary}
\theoremstyle{remark}

\newtheorem{remark}{Remark}[section]

\begin{document}



\twocolumn[

\aistatstitle{Landing Probabilities of Random Walks for Seed-Set Expansion in Hypergraphs}

\aistatsauthor{ Eli Chien \And Pan Li \And  Olgica Milenkovic }

\aistatsaddress{ Department ECE, UIUC \And Department CS, Purdue \And Department ECE, UIUC } ]

\begin{abstract}
We describe the first known mean-field study of landing probabilities for random walks on hypergraphs. In particular, we examine clique-expansion and tensor methods and evaluate their mean-field characteristics over a class of random hypergraph models for the purpose of seed-set community expansion. We describe parameter regimes in which the two methods outperform each other and propose a hybrid expansion method that uses partial clique-expansion to reduce the projection distortion and low-complexity tensor methods applied directly on the partially expanded hypergraphs. \footnote{Eli Chien and Pan Li contribute equally to this work. A short version of this paper appears in ITW 2021.} %We illustrate our results on a number of synthetic datasets.
\end{abstract}
\vspace{-0.1cm}
\section{Introduction}
% \vspace{-0.1cm}
Random walks on graphs are Markov random processes in which given a starting vertex, one moves to a randomly selected neighbor and then repeats the procedure starting from the newly selected vertex~\cite{lovasz1993random}. Random walks are used in many graph-based learning algorithms such as PageRank~\cite{page1999pagerank} and Label Propagating~\cite{zhu2002learning}, and they have found a variety of applications in local community detection~\cite{andersen2006local,gleich2012vertex}, information retrieval~\cite{page1999pagerank} and semi-supervised learning~\cite{zhu2002learning}.

Random walks are also frequently used to characterize the topological structure of graphs via the hitting time of a vertex from a seed, the commute time between two vertices~\cite{von2014hitting} and the mixing time which also characterizes global graph connectivity~\cite{aldous1995reversible}. Recently, a new measure of vertex connectivity and similarity, termed a landing probability (LP), was introduced in~\cite{kloumann2017block}. The LP of a vertex is the probability of a random walk ending at the vertex after making a certain number of steps. Different linear combinations of LPs give rise to different forms of PageRanks (PRs), such as the standard PR~\cite{page1999pagerank} and the heat-kernel PR~\cite{chung2007heat}, both used for various graph clustering tasks. In particular, Kloumann et al.~\cite{kloumann2017block} also initiated the analysis of PRs based on LPs for seed-based community detection. Under the assumption of a generative stochastic block model (SBM)~\cite{holland1983stochastic} with two blocks, the authors of~\cite{kloumann2017block} proved that the empirical average of LPs within the seed community concentrates around a deterministic centroid. Similarly, the empirical averages of LPs outside the seed community also concentrate around another deterministic centroid. These deterministic centroids are the mean-field counterparts of the empirical averages. Kloumann et al.~\cite{kloumann2017block} also showed that the difference of the centroids decays geometrically with a rate that depends on the number of random walk steps and the SBM parameters. The above result implies that the standard PR is optimal for seed-set community detection from the perspective of marginal maximization, provided that only the first-order moments are available.

On the other hand, random walks on hypergraphs (RWoHs) have received significantly less attention in the literature despite the fact that hyperedges more accurately capture higher-order relations between entities when compared to edges. Most of the work on hypergraph clustering has focused on subspace clustering~\cite{agarwal2005beyond}, network motif clustering~\cite{benson2016higher,li2017motif}, ranking data categorization~\cite{li2017inhomogeneous} and heterogeneous network analysis~\cite{yang2018meta}. Random walks on hypergraphs are mostly used indirectly, by replacing hypereges with cliques, merging the cliques and then exploring random walks on standard graphs~\cite{zhou2007learning,chitra2019random}. We refer to this class of approaches as \emph{clique-expansion random walks on hypergraphs} (clique-expansion RWoHs) which were successfully used for community detection in~\cite{yin2017local}. However, it is well-known that clique-expansion can cause significant distortion in the clustering process~\cite{hein2013total,li2018quadratic,chien2018hs}. This motivated parallel studies on higher-order methods which work directly on the hypergraph. Higher-order Markov chain random walk methods were described in~\cite{wu2016general} and shown to have excellent empirical performance; for simplicity, we henceforth refer to this class of walks as \emph{tensor RWoHs}.
%Tensor RWoH empirically show the good performance of tensor RWoH methods \textcolor{olive}{Might need to spice it up.}.
In a different direction, the authors of~\cite{chan2018spectral} defined RWoHs based on a non-linear Laplacian operator whose spectrum carries information about the conductance of a hypergraph. The method of~\cite{chan2018spectral} can also be used to address a number of semi-supervised learning problems on higher-order
data structures by solving convex optimization problems~\cite{li2018quadratic}. Using non-linear Laplacians requires highly non-trivial analytical techniques, and conductance is often not the only relevant performance metric for clustering and community detection. % in stochastic community models.
Furthermore, convex optimization formulations often obscure our theoretical understanding of the underlying problem.%Albeit the widespread utility of RWoHs and their induced PRs, our understanding of RWoH is far from the same level of maturity of that of RWoG.

The focus of this work is on providing the first known characterization of LPs for RWoHs, determining various trade-offs between clique-expansion and tensor RWoHs for the task of seed-set community expansion and proposing means for combining the two methods when appropriate. We adopt a methodology similar to the one used in~\cite{kloumann2017block} for classical graphs: The hypergraphs are assumed to be generated according to a well-studied hypergraph stochastic block model (hSBM)~\cite{ghoshdastidar2017consistency,chien2018community,ahn2018hypergraph,chien2019minimax} and seed-expansion is performed via mean-field analysis of LPs of random walks of different lengths. Our contributions are as follows:
\begin{itemize}
\vspace{-0.2cm}
\item We derive asymptotic results which show that the empirical centroids of LPs concentrate around their mean-field counterparts.
\vspace{-0.1cm}
\item We prove that LPs of clique-expansion RWoHs behave similarly as the LPs of random walks on graphs. More precisely, the difference between the empirical centroids of LPs within and outside the seed community decays geometrically with the number of steps in the random walks.
\vspace{-0.1cm}
\item We show that the LPs of tensor RWoHs behave differently than those corresponding to clique-expanded graphs when the size
of the hyperedges is large: If the hyperedges density within a cluster is at least twice as large as that across clusters, the difference between the empirical centroids of LPs within and outside the seed community converges to a constant dependent on the model parameters. Otherwise, the difference decreases geometrically with the length of the random walk. Consequently, tensor RWoHs exhibit a phase transition phenomenon.
\item As explained in~\cite{kloumann2017block}, combining information about both the first and second moment of the LPs leads to a method that on the SBM performs as well as belief-propagation, which is optimal. We combine this method with LPs of clique-expansion RWosH and tensor RWoHs and show that these two methods have different regimes in which they exhibit good performance; as expected, the regimes depend on the parameter settings of the hSBM. This is due to the fact that LPs of tensor RWoHs has a larger centroid distance while LPs of clique-expansion have smaller (empirical) variance.
\vspace{-0.1cm}
\item We propose a novel hypergraph random walk technique that combines partial clique-expansion with tensor methods. The goal of this method is to simultaneously avoid large distortion introduced by clique-expansion and reduce the complexity of tensor methods by reducing the size of the hyperedges. The method builds upon the theoretical analysis of the means of LPs and \emph{empirical} evidence regarding the variance of the LPs and hence extends the work in~\cite{li2019optimizing}. A direct analysis including the variance of the LPs of the tensor method appears challenging.
%The difference between the  empirical centroids of LPs within and outside a community increase with the number of RW steps if the community hyperedge density is \textcolor{red}{add quantitative formulas, not much larger etc}, while the same value may decay if such condition is not satisfied. Here, we observe a phase transition phenomenon.
\vspace{-0.1cm}
\item The analysis for tensor RWoHs proved to be difficult as it is essentially requires tracking a large number of states in a standard high-order Markov chain. To mitigate this problem and make our analysis tractable, we introduce a novel state reduction strategy which significantly decreases the dimensionality of the problem. This technical contribution may be of independent interest in various tensor analysis problems.
\end{itemize}

The paper is organized as follows. In Section 2, we introduce the relevant notation and formally define the clique-expansion and tensor RWoHs. The same section explains the relationship between LPs and PR methods, and the importance of LPs for seed-set expansion community detection. In Section 3, we introduce the relevant hypergraph SBM, termed $d$-hSBM, and the ideas behind seed-set expansion and the mean-field LPs approach. Theoretical properties of the LPs for clique-expansion and tensor RWoHs are described in Sections 3.2 and 3.3, respectively. In Section 4,
%we present the mean-field analysis for clique-expansion RWoHs, while in Section 5,
we present the mean-field analysis for tensor RWoHs while the same analysis for clique-expansion RWoHs is deferred to the Supplement. We show how to leverage the information provided by the first and second moment of LPs for seed-set expansion in Section 5. Section 6 contains simulation results on synthetic datasets. %\textcolor{red}{can we have some real datasets included as well?}
%and open problems are discussed in Section 8. % Open problems are discussed in Section 7.
% \textcolor{red}{We have way to many acronymes, I suggest getting rid of the word "-based" (it is obvious) and spelling out clique-expansion and tensor, and only keeping RWoH as an acronym.} \textcolor{olive}{Done.}
\vspace{-0.1cm}
\section{Preliminaries}
\vspace{-0.1cm}
%Throughout the paper, we use bold lower-case characters to denote vectors $\mathbf{x}$. We use bold lower-case characters $\mathbf{X}$ for matrices or tensors. The case of tensors is similar. We denote $[C]$ as the integer set $\{1,2,...,C\}$.
\vspace{-0.1cm}
\subsection{Random walks on hypergraphs} \label{sec:RWoH}
\vspace{-0.1cm}
A hypergraph is an ordered pair of sets $G(V,E)$, where $V=\{{v_1,v_2,\ldots,v_n\}}$ is the set of vertices while $E$ is the set of hyperedges.
Each hyperedge $e\in E$ is a subset of $V$, i.e., $e \subseteq V$. Unlike an edge in a graph, a hyperedge $e$ may contain more than two vertices.
If $\forall \, e\in E$ one has $|e| \leq d$, the hypergraph $G$ is termed $d$-bounded. A $d$-bounded hypergraph can be represented by a $d$-dimensional supersymmetric tensor $\mathbf{A}$ such that $A_{v_1,...,v_d} = 1$ if $e = \{v_1,...,v_d\}\in E$, and $A_{v_1,...,v_d} = 0$ otherwise, for all $v_1,\ldots,v_d \in V$. Note that we consider the case where the hyperedges can have repeat vertices are allowed (i.e. multisets). Note that it is easy to extend our analysis to the case where hyperedges cannot have repeat vertices (i.e. sets), albeit the analysis can be more tedious. Henceforth, we assume $G$ to be $d$-bounded with constant $d$, a model justified by numerous practical applications such as subspace clustering~\cite{agarwal2005beyond}, network motif clustering~\cite{li2017inhomogeneous} and natural language processing~\cite{wu2016general}.
% (\textcolor{olive}{Eli: I think it's okay to introduce the definition separately for clique-expansion and tensor RWoH?})
% \textcolor{red}{*We find the following definitions useful for our subsequent derivations. For a standard graph with number of vertices symbol,
% we use A to denote its adjacency matrix. We let x0 denote the initial distribution of the vertex set, and let the k-step transition prob... landing prob...etc}
%\textcolor{red}{stopped here.}
We focus on two known forms of RWoHs.

\textbf{Clique-Expansion RWoHs} is a random walk based on representing the hypergraph via a ``projected'' weighted graph~\cite{chitra2019random,zhou2007learning}: Every hyperedge of $G(V,E)$ is replaced by a clique, resulting in an undirected weighted graph $G^{(ce)}$. The derived weighted graph $G^{(ce)}$ has the same vertex set as the original hypergraph, denoted by $V^{(ce)}=V$. The edge set $E^{(ce)}$ is the union of all the edges in the cliques, with the weight of each $e\in E^{(ce)}$ set to $|\{e'\in E: e\subseteq e'\}|$. The weighted adjacency matrix of $G^{(ce)}$, $\mathbf{A}^{(ce)}$, may be written as $A^{(ce)}_{v_{d-1},v_d} = \sum_{\{v_1,...,v_{d-2}\}\in V}A_{v_1,...,v_{d}}$.

%\textcolor{red}{this is the weighted adjacency matrix, correct?} \textcolor{olive}{Exactly.}
%Correspondingly, the diagonal degree matrix is defined as $\mathbf{D}^{(ce)}$ where $D_{vv}^{(ce)} = \sum_{u\in V}M_{vu}$.
% The number of paths of length $k+1$ steps may be computed as \textcolor{red}{you need to define the initial choice for y and the vector y itself in words; also, you mention LPs below, I suggest that you make these definitions at the beginning so as not to repeat them for every RW strategy. I added a sentence where you should write this down} \textcolor{olive}{Again, I think it's more nature to introduce the initial distribution $x^{(0)}_{ce}$ or $x^{(0)}_t$ after introducing the random walk mechanism.}
% %we propose to perform ``post-normalized'' random walk on this graph to obtain landing probabilities. The definition of post-normalized random walk is as following
Let $y_{ce}^{(0)}\in [0,1]^{|V|}$ be the initial state vector describing which vertices may be used as the origins or seeds of the random walk and with what probability.
The $(k+1)$-th step random walk state vector equals
\begin{align}\label{eq:rw_CE}
   y_{ce}^{(k+1)} =y_{ce}^{(k)}\mathbf{A}^{(ce)},
\end{align}
% \textcolor{red}{the notation is weird, you now have the ce in the subscript?}(\textcolor{olive}{Eli: I can put all the ce in the superscript if you like.})
while the $k$-step LP of a vertex $v$ in the clique-expansion framework is defined as
$$x_{v;ce}^{(k)} = y_{v;ce}^{(k)}/\|y_{ce}^{(k)}\|_1.$$
% It is known that clique-expansion introduces large distortions when the hyperedges have large cardinality~\cite{li2017motif,chien2018hs}.

\textbf{Tensor RWoHs} are described by a tensor $\mathbf{A}$ corresponding to a Markov Chain of order $d-1$~\cite{wu2016general}. Each step of the walk is determined by the previous $d-1$ states and we use $y^{(k)}_{v_1,...,v_{d-1};t}$ to denote the number of paths of length $k$ whose last $d-1$ visited vertices equal $v_1, v_2, ..., v_{d-1}$. The number of paths of length $k+1$ steps may be computed according to the following expression:
\begin{align}\label{eq:rw_HG}
    y^{(k+1)}_{v_2,...,v_{d};t} = \sum_{v_1=1}^{n} A_{v_1,...,v_{d}}  y^{(k)}_{v_1,...,v_{d-1};t}.
\end{align}
%where $D_{v_1,...,v_{d-1}} = |\{e\in E: \{v_1,...,v_{d-1}\}\subseteq e\}|$.
% \textcolor{red}{usually the labeling goes with increasing the index, why v1 and not vd?}(\textcolor{olive}{Eli: since I define the random walk histroy to be ($v_2,...,v_{d}$) where $v_d$ is the latest vertex we visit here. This should agree with the common definition that the most right index indicate the latest state.})
The $k$-step LP of a vertex $v$ may be defined similarly as that of clique-expansion RWoHs,
$$x_{v, t}^{(k)} = \sum_{v_1,...,v_{d-2}}y_{v_1,...,v_{d-2}, v;t}^{(k)}/\|y_t^{(k)}\|_1.$$
The complexity of computing a one-step LP in a tensor RWoHs equals $O(n^d)$, while the used storage space equals $O(n^{d-1})$. In contrast, computing the one-step LP of a clique-expansion RWoHs has complexity $O(n^2)$ and it requires storage space equal to $O(n)$. To mitigate the computational and storage issues associated with tensors, one may use tensor approximation methods~\cite{gleich2015multilinear,benson2017spacey}; unfortunately, it is not well-understood theoretically how these approximations perform on various learning tasks.

In what follows, whenever clear from the context, we omit the subscripts indicating if the method uses clique-expansion or tensors, and write $x_{v}^{(k)}$ for either of the two types of LPs.
%Note that in this work, for ease of analysis, we adopt the same strategy of \cite{kloumann2017block} and consider a slightly different RWs whose LPs depend on the number of paths.

%Note that in this work, for ease of analysis, we adopt the same strategy of \cite{kloumann2017block} and po Note that in this work, for ease of analysis, we adopt the strategy used in \cite{kloumann2017block} and post-normalized RWs.
\vspace{-0.1cm}
\subsection{Seed-set expansion based on LPs}
\vspace{-0.1cm}
Seed-set expansion is a clustering problem which aims to identify subsets of vertices around seeds that are densely connected among themselves~\cite{xie2013overlapping,gleich2012vertex,kloumann2017block}. Seed-set expansion may be seen as a special form of local community detection, and some recent works~\cite{chien2018community,chien2019minimax,ahn2018hypergraph,paul2018higher,angelini2015spectral,kim2018stochastic} has also addressed community detection in hypergraphs using approaches that range from information theory to statistical physics.

% In what follows, we analyze the LPs for community detection in hypergraphs by using the methodology outlined in~\cite{kloumann2017block}.
% (\textcolor{olive}{Eli: We should delete the below sentences since now our definition of RW is exactly the same as~\cite{kloumann2017block})
% We slightly depart from the above approach by using LPs that are proportional to the numbers of $k$-step paths from a seed vertex to a vertex $v$. Note that this definition of RWs is equivalent to the original definition~\cite{kloumann2017block} when the hypergraph is regular.}
Seed-set expansion community detection algorithms operate as follows: One starts from a seed set within one community of interest and performs a random walk. Since vertices within the community are densely connected, the values of the LPs of vertices within the community are in general higher than those of vertices outside of the community. Consequently, thresholding properly combined LP values may allow for classifying vertices as being inside or outside of the community. Formally, each vertex $v$ in a hypergraph $G(V,E)$ is associated with a vector of LPs $(x_v^{(0)},x_v^{(1)},...)$ of all possible lengths. The generalized Page Rank (GPR) of a vertex $v$ with respect to a pre-specified set of weights $(\gamma_k)_{k=0}^{\infty}$ is defined as $\sum_{k=0}^{\infty}\gamma_k x_v^{(k)}$. The GPRs of vertices are compared to a threshold to determine whether they belong to the community of interest. Consequently, GPRs lead to linear classifiers that use LPs as vertex features. The above described GPR formulation includes Personalized PR (PPR)~\cite{andersen2006local}, where $\gamma_k = (1-\alpha)\alpha^k$, and heat-kernal PR (HPR)~\cite{chung2007heat}, where $\gamma_k = e^{-h}h^k/k!$, for properly chosen $\alpha,\, h$.
%This heuristic provides an important usage of various PR algorithms such as Personalized PR~\cite{andersen2006local} and heat-kernal PR~\cite{chung2007heat}.

An important question that arises in seed-set expansion is how to choose the weights of the GPR in order to insure near-optimal or optimal classification~\cite{kloumann2017block}. To this end, start with a partition into two communities $V_0,V_1$ of $V$.  Let $\mathbf{a} = (a^{(0)},a^{(1)},...)$ denote the arithmetic mean (centroid) of the LPs of vertices $v\in V_0$, $a^{(k)} \eqDef \frac{1}{|V_0|}\sum_{v\in V_0}x_v^{(k)}$, and let $\mathbf{b} = (b^{(0)},b^{(1)},...)$ denote the arithmetic mean (centroid) of the LPs of vertices $v\in V_1$ , $b^{(k)} \eqDef \frac{1}{|V_1|}\sum_{v\in V_1}x_v^{(k)}$. If the only available information about the distribution of the LPs are $\mathbf{a}$ and $\mathbf{b}$, a discriminant with weights $\gamma_k = a^{(k)} - b^{(k)}$ is optimal since the deterministic boundary is orthogonal to the line that connects the centroids of the two communities. Klouman et al.~\cite{kloumann2017block} observed that for community detection over graphs generated by standard SBMs~\cite{holland1983stochastic}, such a discriminant corresponds to PPR with an adequately chosen parameter $\alpha$.

% if we focus on the first The geometric discriminant function $f(\mathbf{x}_v) = (\mathbf{a}-\mathbf{b})^T\mathbf{x}_v$ will rank each node $v$ based on its landing probability vector. This can be related to the generalized PageRank~\cite{gleich2015pagerank}. The generalized PageRank assigns scores according to the infinite sum $\sum_{c=0}^{\infty}\gamma_c x_v^{(c)}$ for some infinite dimension vector $\mathbf{\gamma}$ with unit $l_1$-norm. That is, $\forall \mathbf{\gamma} \in \Gamma$, $\sum_{c = 0}^{\infty}\gamma_c = 1$. Personalized PageRank~\cite{page1999pagerank} and heat kernal PageRank~\cite{chung2007heat} are the two special cases of this generalized PageRank. For Heat kernal PageRank we have $\gamma_c = h^c\frac{e^{-h}}{c!}$ for some $h\geq 0$ and for personalized PageRank we have $\gamma_c = (1-\alpha)\alpha^c$ for some $\alpha \in (-1,1)$. Note that usually the parameter $\alpha$ in personalized PageRank is interpreted as the ``teleportation'' probability which requires $\alpha$ to be non-negative. However, under the above interpretation, the personalized PageRank is well-defined for $\alpha \in (-1,1)$.

In what follows we study the statistical properties of the centroids $a^{(k)}$ and $b^{(k)}$ of RWoHs, where the hypergraphs are generated by a hSBM. The main goal of the analysis is to characterize the centroid difference $a^{(k)} - b^{(k)}$ which guides the choice of the weights $\gamma_k$. Some results related to the variance of the landing probabilities and comparisons of the discriminative power of the two types of LPs will be presented as well.

\vspace{-0.1cm}
\section{Statistical characterization of LPs}
\vspace{-0.1cm}
We start by introducing the $d$-hSBM of interest. Afterwards, we outline the mean-field approach for our analysis and use the obtained results to determine the statistical properties of LPs of clique-expansion and tensor RWoHs. In particular, we provide new concentration results for the corresponding LPs.

For notational simplicity, we focus on symmetric hSBMs with two blocks only. More general models may be analyzed using similar techniques.
\begin{definition}[$d$-hSBM]
    The $d$-hSBM$(n, p,q)$ is a $d$-bounded hypergraph $G(V,E)$ such that $\forall e\in E, |e| \leq d$ and $|V| = n$. The hypergraph has the following properties. Let $\sigma$ be a binary labeling function $\sigma: V \mapsto \{0, 1\}$, which induces a partition of $V = V_0 \cup V_1$ where $V_i = \left\{v\in V_i : \sigma(v) = i\right\}$ and $|V_{2-i}| = \lceil n/2 \rceil $ or $|V_{2-i}| = \lfloor n/2 \rfloor$, for $i=1,2$. The hypergraph $G(V,E)$ is uniquely represented by an adjacency tensor $\bf{A}$ of dimension $d$, where for all indices $v_1\leq ... \leq v_d \in V$, $A_{v_1,...,v_d}$ are i.i.d. Bernoulli random variables and $\mathbf{A}$ is symmetric.
    \begin{equation*}
       \mathbb{P}\left (A_{v_1,...,v_d} = 1\right ) =
        \begin{cases}
            p, & \text{if } \sigma(v_1) = ... = \sigma(v_d)\\
            q, & \text{otherwise},
        \end{cases}
    \end{equation*}
where $0 < q <p \leq 1$. In our subsequent asymptotic analysis for which $n\rightarrow \infty$, we assume that $\frac{p}{q} = \Theta(1)$ is a constant. This captures the regime of parameter values for which the problem is challenging to solve.
%Observe that sampling from the $d$-hSBM may not produce strictly $d$-uniform hypergraphs. Nevertheless, given that $d$ is a constant and $n$ is allowed to grow, the probability of observing non-uniform instances is small.
\end{definition}
\vspace{-0.1cm}
\subsection{Mean-field LPs}
\vspace{-0.1cm}
Next we perform a mean-field analysis of our model in which the random hypergraph topology is replaced by its expected topology.
This results in $\mathbf{A}$ and the clique-expansion matrix $\mathbf{A}^{(ce)}$ being replaced by $\E\mathbf{A}$ and $\E\mathbf{A}^{(ce)},$ respectively.

The mean-field values of the LPs are defined as follows: For clique-expansion RWoHs, the mean-field counterpart of~\eqref{eq:rw_CE} equals
\begin{align}\label{eq:rw_CE_MF}
   \bar{y}_{ce}^{(k+1)} =\bar{y}_{ce}^{(k)} \E\mathbf{A}^{(ce)},
\end{align}
and the corresponding mean-field of a $k$-step LP for vertex $v$ reads as $\bar{x}_{v;ce}^{(k)} = \bar{y}_{v;ce}^{(k)}/\|\bar{y}_{ce}^{(k)}\|_1$.
For tensor RWoHs, the mean-field counterpart of~\eqref{eq:rw_HG} equals
\begin{align}\label{eq:rw_HG_MF}
 %  \bar{y}^{(c+1)}_{i_2,...,i_{d}} = \frac{1}{\langle\,\E \mathbf{D},\bar{\mathbf{y}}^{(c)}\rangle} \sum_{i_1=1}^{n} \E A_{i_1,...,i_{d}} \bar{y}^{(c)}_{i_1,...,i_{d-1}}.
     \bar{y}^{(k+1)}_{v_2,...,v_{d};t} = \sum_{v_1=1}^{n} \bar{y}^{(k)}_{v_1,...,v_{d-1};t}  \E A_{v_1,...,v_{d}} .
\end{align}
%\textcolor{red}{notation is weird, as you know assume in the equation above that the vertices are labeled from 1 to n; in other places you did not have that.}
The $k$-step LP of a vertex $v$ equals $\bar{x}_{v;t}^{(k)} = \sum_{v_1,...,v_{d-2}}\bar{y}_{v_1,...,v_{d-2}, v;t}^{(k)}/\|\bar{y}_t^{(k)}\|_1$. %In both cases, one can view $\bar{x}^{(k)}_{v}$ as the LPs of RWs over the hypergraphs generated from $d$-hSBM in expectation.
For non-degenerate random variables of interest in our study, $\bar{x}^{(k)}_{v}\neq \E x^{(k)}_v$, but one can nevertheless show that the geometric centroids of the LPs $a^{(k)}$ and $b^{(k)}$ concentrate around their mean-field counterparts $\bar{a}^{(k)}$ and $\bar{b}^{(k)}$, respectively. This concentration result guarantees
consistency of our method.
\vspace{-0.1cm}
\subsection{Concentration results}\label{subsec:CE}
\vspace{-0.1cm}
The mean-field of the LPs for the $d$-hSBM$(n,p,q)$ model in the clique-expansion setting is described in the following theorem. %Recall the clique-expansion strategy that projects $G$ onto a graph $G^{(ce)}$ (Section ~\ref{sec:RWoH}).
\begin{theorem}\label{MAINTHM:CE}
Let $G$ be sampled from a $d$-hSBM$(n,p,q)$ model and let $G^{(ce)}$ be the graph obtained from $G$ through clique-expansion.
Let the initial state vector of the RWoHs be $y^{(0)}_{s; ce} = 1$ and $y^{(0)}_{v;ce} = 0$ otherwise, where $s$ is a vertex chosen uniformly at random from $V_0$. Set
$\bar{y}^{(0)}_{ce} = \E y^{(0)}_{ce}$. Then for all $k\geq 0$ we have
    \begin{align*}
        \bar{x}_{v; ce}^{(k)} =
        \begin{cases}
            \bar{a}^{(k)} &\text{if }v\in V_0\\
            \bar{b}^{(k)} &\text{if }v\in V_1
        \end{cases},
    \end{align*}
    where $\bar{a},\bar{b}$ satisfy the following recurrence relation
    \begin{align}
    &\begin{bmatrix}
    \bar{a}^{(k)}\\
    \bar{b}^{(k)}
    \end{bmatrix}
    =
    \begin{bmatrix}
       \frac{p+(2^{d-2}-1)q}{p+(2^{d-1}-1)q}  &   \frac{2^{d-2}q}{p+(2^{d-1}-1)q}\\
       \frac{2^{d-2}q}{p+(2^{d-1}-1)q}  &   \frac{p+(2^{d-2}-1)q}{p+(2^{d-1}-1)q}
    \end{bmatrix}
    \begin{bmatrix}
    \bar{a}^{(k-1)}\\
    \bar{b}^{(k-1)}
    \end{bmatrix}
    , \nonumber\\
    &\begin{bmatrix}
    \bar{a}^{(0)}\\
    \bar{b}^{(0)}
    \end{bmatrix}
    = \frac{2}{n}
    \begin{bmatrix}
    1\\
    0
    \end{bmatrix}. \label{eq:CE1}
\end{align}
%    and the centroid distance of the $k^{th}$ step will be
%    \begin{align*}
%       \bar{a}^{(k)} - \bar{b}^{(k)} = \frac{2}{n}(\frac{p-q}{p+3q})^{k} \;\forall k\geq 0.
%    \end{align*}
\end{theorem}
\begin{remark}
The eigenvalue decomposition leads to
    \begin{align*}
     \bar{a}^{(k)} - \bar{b}^{(k)} = \frac{2}{n}\left[\frac{p-q}{p+(2^{d-1}-1)q}\right]^{k}, \;\forall k\geq 0.
    \end{align*}
This result reveals that the geometric discriminant under the $d$-hSBM$(n,p,q)$ is of the same form as that of PPR with parameter $\alpha = \frac{p-q}{p+(2^{d-1}-1)q}$. The result is also consistent with the finding for the special case $d = 2$ described in~\cite{kloumann2017block}.
\end{remark}

Next we show that the geometric centroids of LPs of clique-expansion RWoHs will asymptotically concentrate around their mean-field counterparts, which establishes consistency of the mean-field analysis.
%Next, we show that the geometric centroids of LPs of CE-RWoH will concentrate to their mean-field compliments asymptotically by the similar argument akin to \cite{kloumann2017block}. This can be viewed as the consistency of the mean-field analysis.
%\textcolor{red}{The result is only true for $d=3$. We have nothing more general than this? One may argue that this is a special case for CE...} \textcolor{olive}{Done.}
\begin{lemma}\label{lma:concentrate}
    Assume that $G$ is sampled from a $d$-hSBM$(n,p,q)$ model, for some constant $d\geq 3$. Let $x_{v;ce}^{(k)}$ be the LPs of a clique-expansion RWoHs on $G^{(ce)}$ satisfying~\eqref{eq:rw_CE}. Also assume that $\frac{n^{d-1}q^2}{\log n}\rightarrow \infty$. Then,
    for any constant $\epsilon>0$, $n$ sufficiently large and a bounded constant $k\geq 0$, one has
    \begin{align*}
        & a^{(k)} \eqDef \frac{1}{|V_0|}\sum_{v\in V_0}x_{v;ce}^{(k)} \in [(1-\epsilon)\bar{a}^{(k)},(1+\epsilon)\bar{a}^{(k)}]\\
        & b^{(k)} \eqDef \frac{1}{|V_1|}\sum_{v\in V_1}x_{v;ce}^{(k)} \in [(1-\epsilon)\bar{b}^{(k)},(1+\epsilon)\bar{b}^{(k)}],
    \end{align*}
    with probability at least $1-o(1)$. %\textcolor{red}{Something is fishy here, no K in the formulas, does $\epsilon$ depend on $k$ etc?}
\end{lemma}
% The proof of Theorem~\ref{MAINTHM:CE} is presented in Section~\ref{sec:CEhRW}, while the proof of Lemma~\ref{lma:concentrate} is deferred to the Supplement~\ref{app:pflma1}.
The proof of Theorem~\ref{MAINTHM:CE} and Lemma~\ref{lma:concentrate} are presented in Supplement~\ref{sec:CEhRW} and \ref{app:pflma1} respectively.

%===============================================================================================================
%\subsection{Concentration results for LPs of tensor RWoH}\label{sec:highorderPRonhSBM}
In the tensor setting, one can also determine the distance between the centroids of LPs based on a recurrence relation.
However, a direct application of this method requires tracking $2^{d-1}$ states in the recurrence which makes the analysis intractable.
%To see this, observing that now we have to analyze $\bar{y}_{v_1,...,v_{d-1};t}^{(k)}$ where each $v_i$ may belong to $V_0$ or $V_1$. Follows the similar analysis to the case of clique-expansion RWoH, we know that there show be $2^{d-1}$ states for $\bar{y}$.
To address this issue, we introduce a new state reduction technique which allows us to track only $d-1$ states. The key insight used in our proof is that our goal is to characterize the distance between the centroids instead of $\bar{y}$ itself, and that the distance changes are dictated by a significantly smaller state-space recurrence relation.
The state reduction technique also allows us to describe the centroid distance in closed form for $d\leq 5$, as it arises as the solution of a polynomial equation.
Moreover, for large $d$, we justify the use of a heuristic approximation for the centroid distance and verify its quality through extensive numerical simulations.

\begin{theorem}\label{thm:abapprox}
    Let $G$ be sampled from a $d$-hSBM$(n,p,q)$ model with $d\geq 3$ and set the initial vector of the Tensor RWoHs to $y^{(0)}_{s_1,...,s_{d-1}; t} = 1$ and
    $y^{(0)}_{v_1...,v_{d-1};t} = 0$ otherwise, where $s_1,...,s_{d-1}$ are chosen independently and uniformly at random from $V_0$.
    Furthermore, let $\bar{y}^{(0)}_{t} = \E y^{(0)}_{t}$. Then
    \begin{align*}
        \bar{w}_k = \bar{a}^{(k)} - \bar{b}^{(k)} = \frac{2}{n}\frac{\beta_1(k)}{\zeta_1(k)},
    \end{align*}
    where $\beta_1(k)$ and $\zeta_1(k)$ satisfy the following recurrence relations:
    \small
\begin{equation}\label{thmeq:beta}
	\begin{bmatrix}
		\beta_1(k)\\
        \vdots\\
        \beta_{d-1}(k)
	\end{bmatrix}
     = \frac{n}{2}
     \begin{bmatrix}
     	0 & \cdots & 0 & 0 & p-q\\
        q & 0 & \cdots & 0 & p-q\\
        0 & \ddots & \vdots & 0 & p-q\\
        0 & \cdots & q & 0 & p-q\\
        0 & \cdots & 0 & q & p-q\\
     \end{bmatrix}
     \begin{bmatrix}
		\beta_1(k-1)\\
        \vdots\\
        \beta_{d-1}(k-1)
	\end{bmatrix},
\end{equation}
\normalsize
and
\small
\begin{equation}\label{thmeq:zeta}
	\begin{bmatrix}
		\zeta_1(k)\\
        \vdots\\
        \zeta_{d-1}(k)
	\end{bmatrix}
     = \frac{n}{2}
     \begin{bmatrix}
     	2q & \cdots & 0 & 0 & p-q\\
        q & 0 & \cdots & 0 & p-q\\
        0 & \ddots & \vdots & 0 & p-q\\
        0 & \cdots & q & 0 & p-q\\
        0 & \cdots & 0 & q & p-q\\
     \end{bmatrix}
     \begin{bmatrix}
		\zeta_1(k-1)\\
        \vdots\\
        \zeta_{d-1}(k-1)
	\end{bmatrix}.
\end{equation}
\normalsize
The initial conditions take the form
\begin{align*}
  & \begin{bmatrix}
		\zeta_1(0)\\
        \vdots\\
        \zeta_{d-1}(0)
	\end{bmatrix} =
     \begin{bmatrix}
		\beta_1(0)\\
        \vdots\\
        \beta_{d-1}(0)
	\end{bmatrix} = \frac{4}{n^2}
    \begin{bmatrix}
		1\\
        \vdots\\
        1
	\end{bmatrix}.
\end{align*}
\end{theorem}
A closed-form expression for the distance between the centroids may be obtained through eigenvalue decomposition of the matrices specifying the recurrence for
$\beta,\zeta$. This is demonstrated for $d=3$ in Supplement~\ref{app:d3case}. The Abel-Ruffini theorem~\cite{abel1824memoire} establishes that there are no algebraic solutions in terms of the radicals for arbitrary polynomial equations of degree $\geq 5$, which implies that our centroids distance may not have a closed form unless $d-1< 5$.

For our subsequent analysis, we find the following corollary of Theorem~\ref{thm:abapprox} useful.
\begin{theorem}\label{cor:HgeqCE}
    For all $d$-bounded hypergraphs with $d\geq 3$ and for all $k\geq 1$, the centroid distance for the $d$-hSBM$(n,p,q)$ model satisfies
    $$\bar{a}^{(k)} - \bar{b}^{(k)} =\bar{w}_k \geq \frac{p-q}{p+q}\, \bar{w}_{k-1} =\frac{p-q}{p+q}\, (\bar{a}^{(k-1)} - \bar{b}^{(k-1)}).$$
\end{theorem}
Combining the above result with that of Theorem~\ref{MAINTHM:CE} reveals that the distance between the centroids of the tensor RWoHs is greater than that of the clique-expansion RWoHs whenever $d\geq 3$. Applying a telescoping sum on the result of Theorem~\ref{cor:HgeqCE} also produces the following bound
$$\bar{w}_k \geq  \frac{2}{n}\left( \frac{p-q}{p+q}\right )^k.$$
Comparing this bound to the result of Theorem~\ref{MAINTHM:CE}, one can observe that the centroid distance $\bar{w}_k$ of the tensor RWoHs decays slower than that of the clique-expansion RWoHs with increasing $k$, and that the centroid distance of LPs of tensor RWoHs is larger than that of clique-expansion RWoHs.
We defer the proof of the results to Supplement~\ref{sec:pfthm3} and instead present simulation results in Figure~\ref{fig:simul1}.
\begin{figure}[!htb]
\centering
  \includegraphics[width=0.49\linewidth,height=0.3\linewidth]{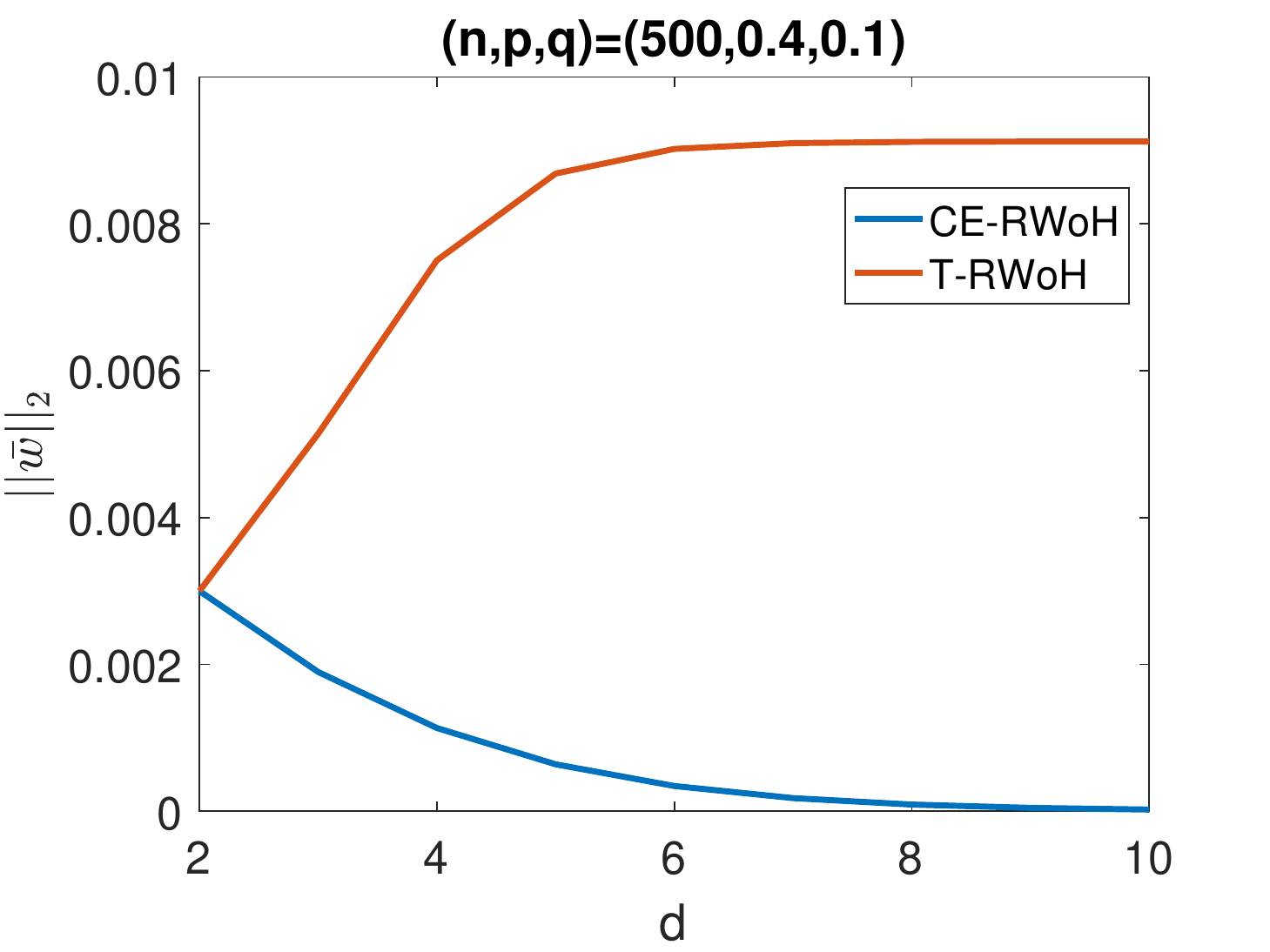}
  \includegraphics[width=0.49\linewidth,height=0.3\linewidth]{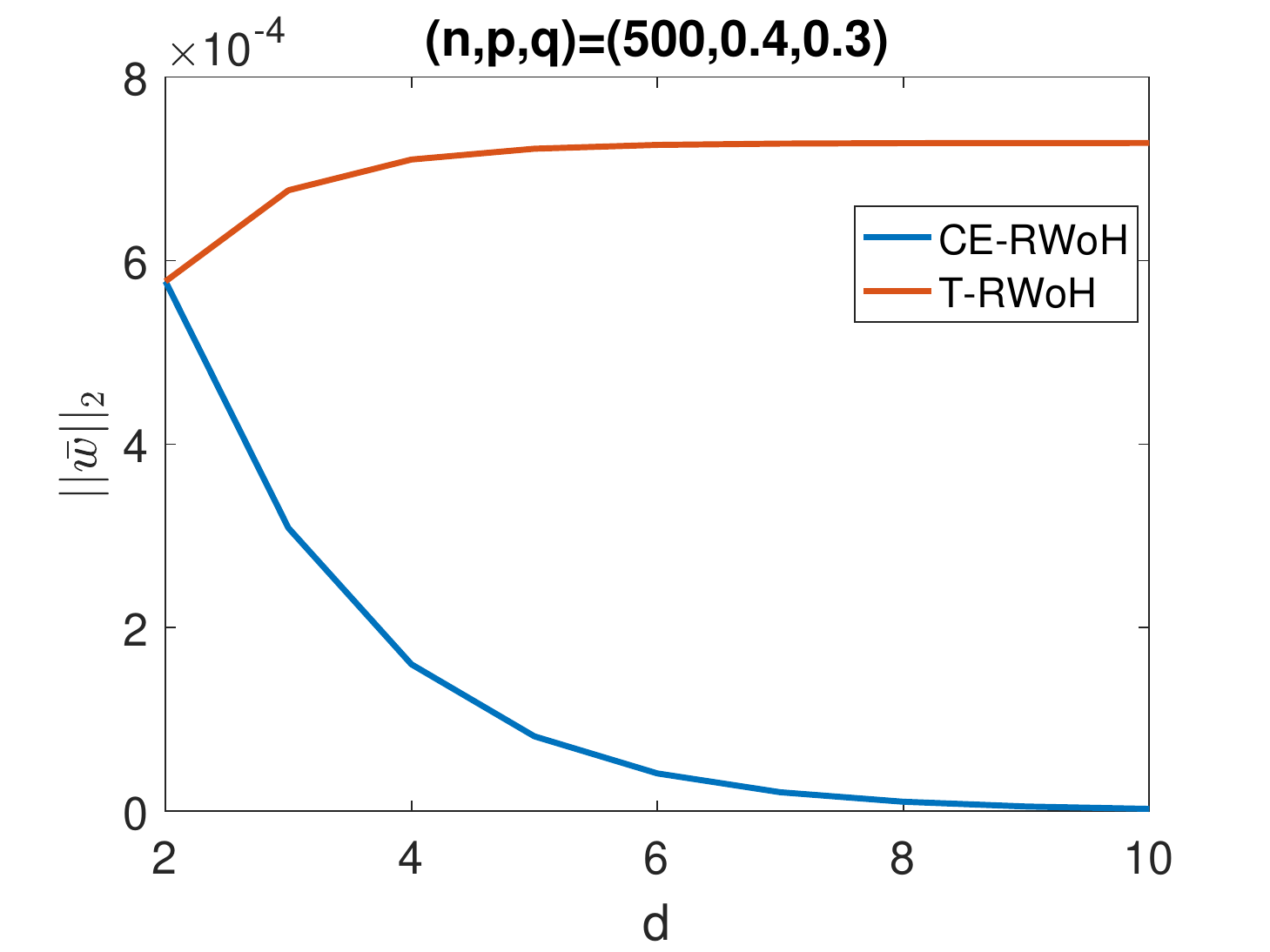}
  \vspace{-0.7cm}
  \caption{Centroid distances for the two studied RWoHs. We ran $20$ steps of the random walk and used Theorems~\ref{MAINTHM:CE},~\ref{thm:abapprox} to calculate the centroid distance $||\bar{w}||_2 = \sqrt{\sum_{k=1}^{20}\bar{w}_k^2}$. %For $d=2$, we set the centroid distance of tensor RWoH to that corresponding to the clique-expansion RWoH, as our tensor results only hold for $d\geq 3$.
  }\label{fig:simul1}
  \vspace{-0.3cm}
\end{figure}
\begin{figure}[!htb]
\centering
  \includegraphics[width=0.31\linewidth]{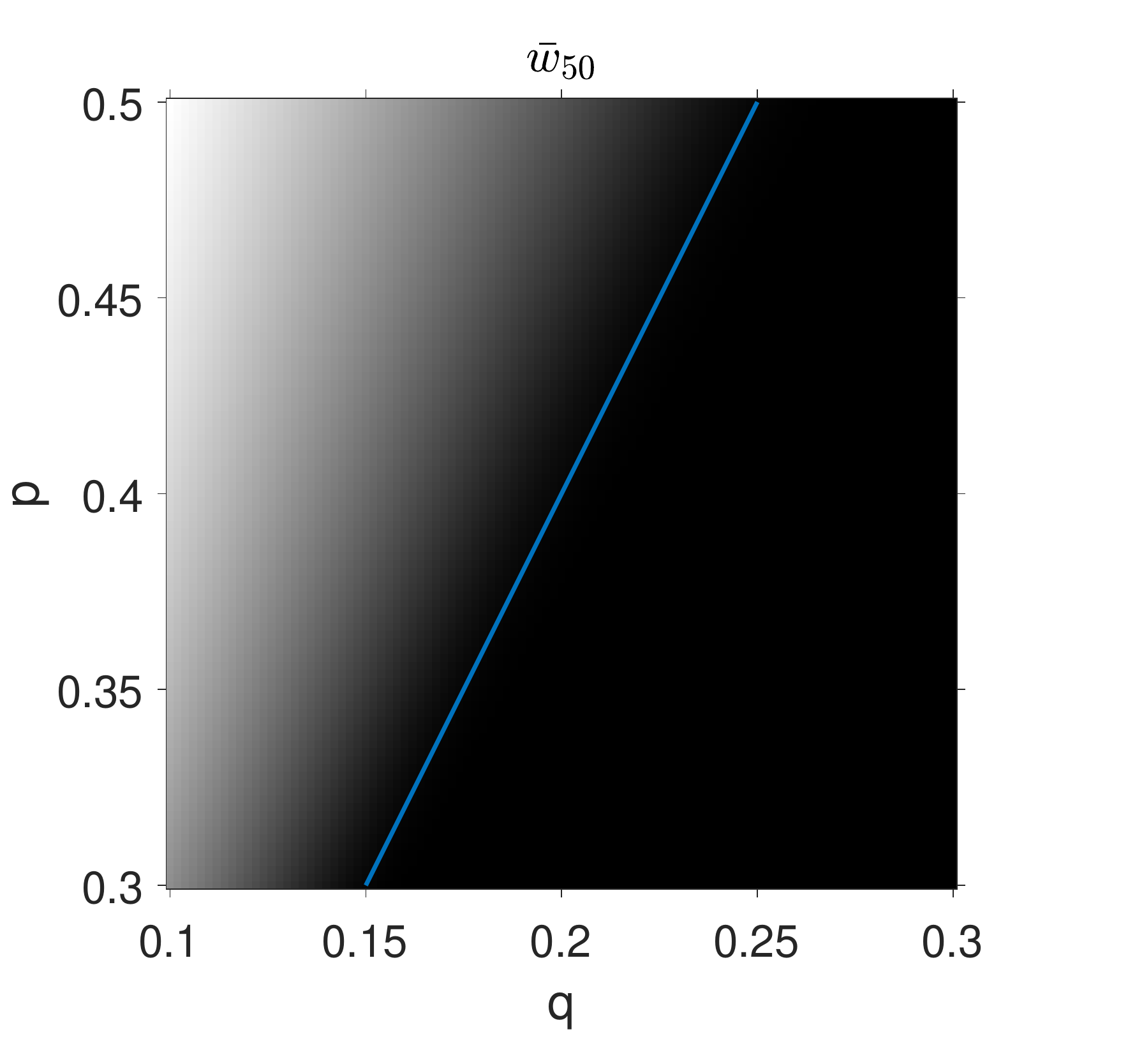}
  \includegraphics[width=0.33\linewidth]{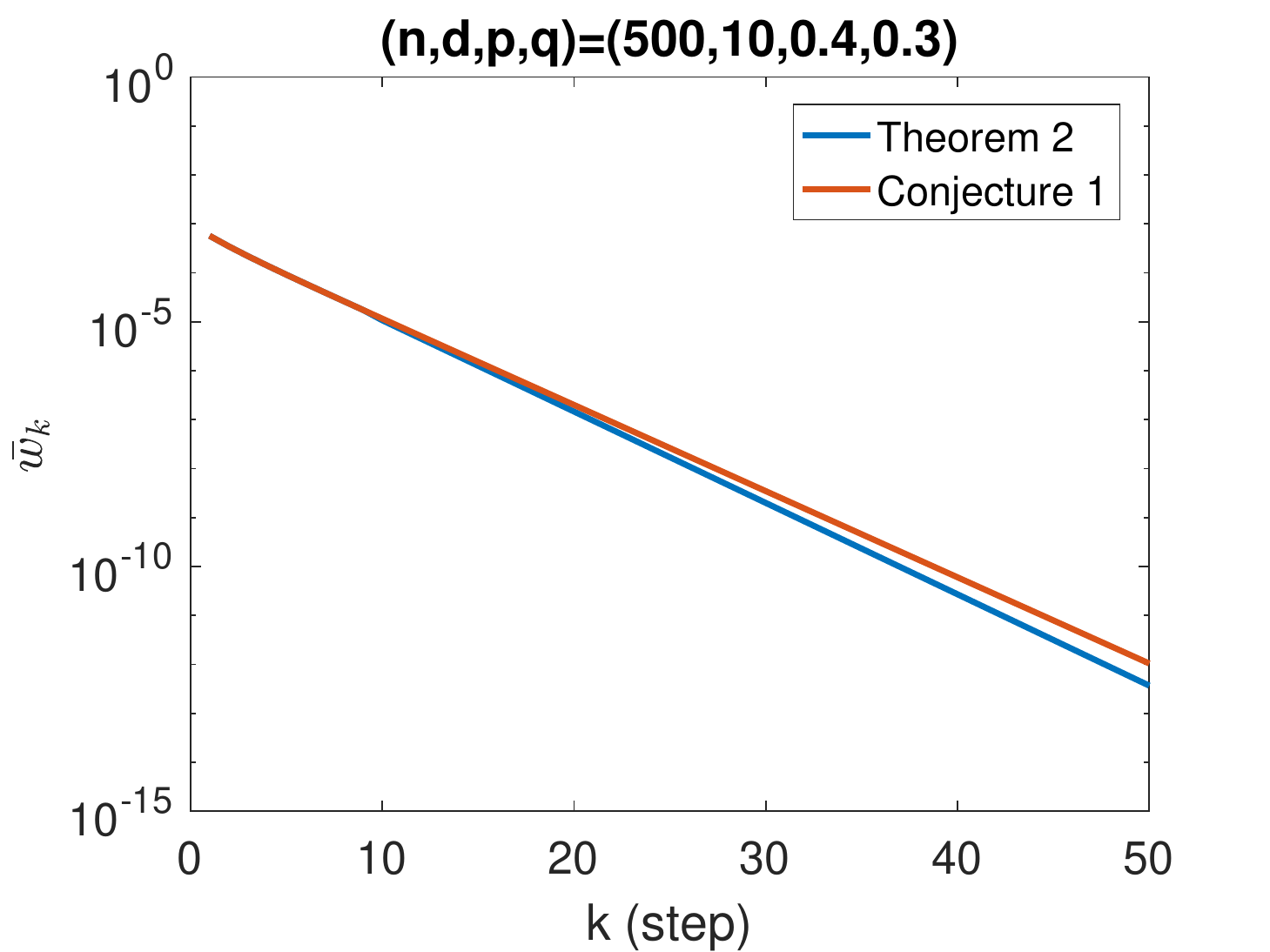}
  \includegraphics[width=0.33\linewidth]{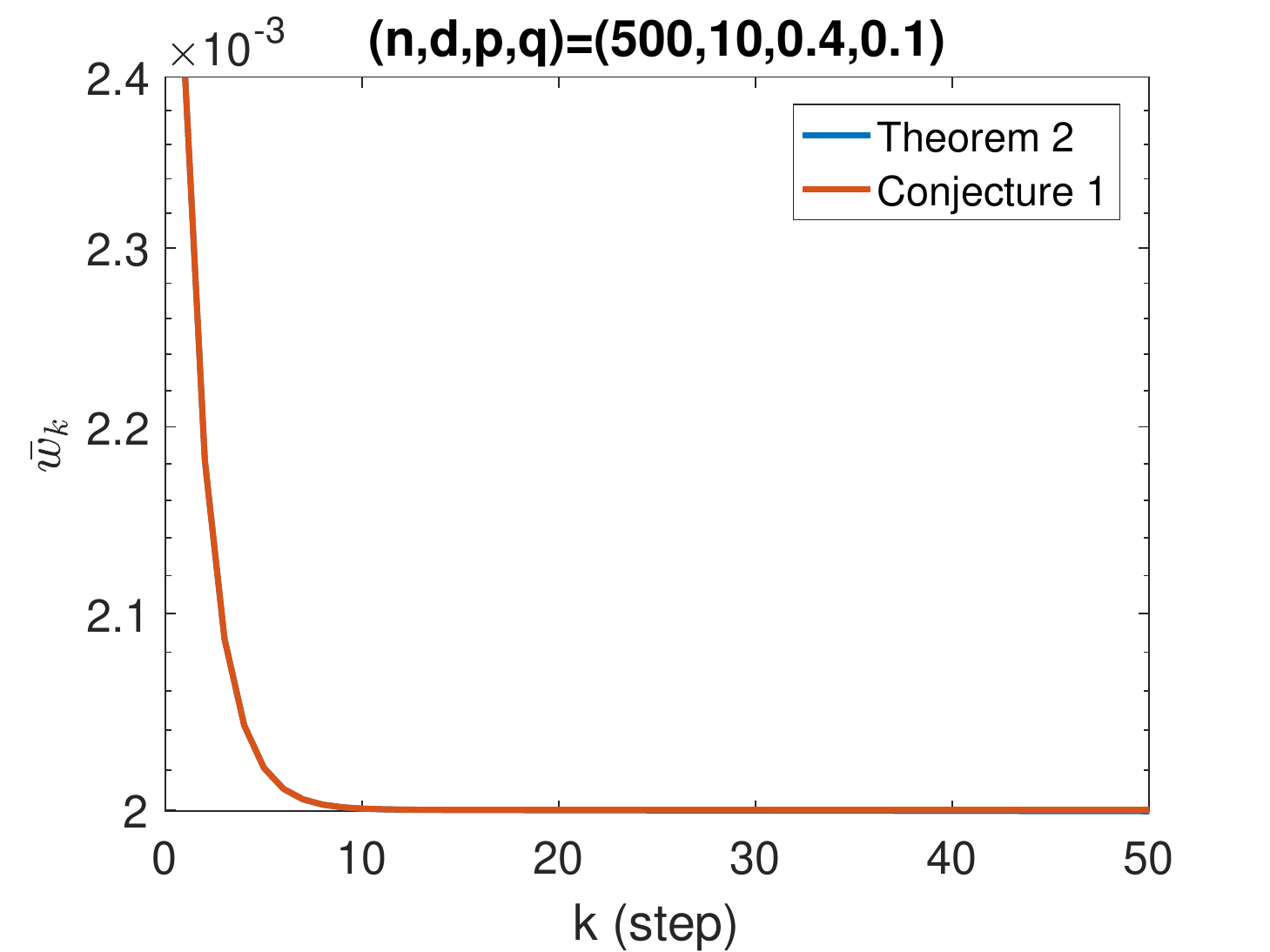}
  \vspace{-0.7cm}
  \caption{The phase transition following from Theorem~\ref{thm:abapprox}. (Right) The gray scale captures the magnitude of $\bar{w}_{50}$ for different values of $(p,q)$. The darker the shade, the smaller the $\bar{w}_{50}$; the separating line is $p=2q$. (Middle, Left) The decay of $\bar{w}_{k}$ for $(p,q) = (0.4,0.3)$ and $(0.4,0.1)$, respectively. }\label{fig:simul2}
  \vspace{-0.3cm}
\end{figure}
%\textcolor{blue}{Is there some intuition on the above result. Please either add proof or simulation showing the relation to CE.} \textcolor{red}{Left for the proofs from Pan.}

%\textcolor{red}{
%In order to prove Theorem~\ref{thm:d3case} \textcolor{blue}{Not prove Thm2 but to derive the the closed form of $\zeta_1(k), \beta_1(k)$. Please be specific} , we need to obtain the close form of eigenvalues of matrices of size $(d-1)\times (d-1)$. \textcolor{blue}{(Please give a smooth introduction why the problem reduces to solving a polynomial)}. It is equivalent to solve the characteristic polynomial of degree $d-1$. However, by Abel - Ruffini theorem \textcolor{blue}{reference} we know that there is no algebraic solution for general polynomials of degree greater or equal to $5$. Hence we can only generalize our Theorem~\ref{thm:d3case} directly for the case $d \leq 5$. \textcolor{blue}{Please say it from a positive angle.}}
%\textcolor{blue}{Explain $d = 4, 5$ case. $d = 6$ is not possible because ... }  .
Although there may not be a general closed form characterization for the centroid distance when $d\geq 6$, we may still obtain some simple approximation results for the case that $d$ is sufficiently large by analyzing the characteristic polynomials of $\beta,\zeta$. For sufficiently large $d$ %and $|\frac{p-q}{q}-1|>\epsilon$ for some constant $\epsilon>0$
, we have the following approximation for $\bar{w}_k$, $\frac{2}{n}\frac{C_3p^k}{C_1(2q)^k+C_2p^k}$,
where $C_1,C_2$ and $C_3$ are constants independent of $k$. We describe how this heuristic naturally arises from the characteristic polynomial of the recurrence and how it is supported by extensive simulations in Supplement~\ref{app:thm2larged}.

We also observe that there exists a phase transition at $p = 2q$, illustrated in Figure~\ref{fig:simul2}. When $p>2q$, we have $\bar{w}_k \rightarrow \frac{2}{n} \frac{C_3}{C_2}$. This implies that the centroid distance does not diminish when increasing the step size $k$. Thus, for this parameter setting the tensor RWoHs behaves fundamentally different from the clique-expansion RWoHs (see Theorem~\ref{MAINTHM:CE}). On the other hand, if $q<p<2q$ then $\bar{w}_k$ decays roughly geometrically similarly to what is proved in Theorem~\ref{MAINTHM:CE}. We conjecture that the constants are as listed below.
\begin{conjecture}\label{conj:1}
    % Assume that $p>q$, $|\frac{p-q}{q}-1|>\epsilon$ for some constant $\epsilon>0$ and that $d$ is sufficiently large. For the same initialization $y^{(0)}$ as described in Theorem~\ref{thm:abapprox} we have
    \begin{align*}
        & C_1 = \frac{-q}{p-2q},\;C_2 = \frac{p-q}{p-2q},\;C_3 = \frac{p-q}{p}.
    \end{align*}
\end{conjecture}

Figure~\ref{fig:simul2} also shows that for $p>2q$ the centroid distance decays very slowly with $k$ and that the difference between our conjectured behavior and the result of the pertinent theorem is very small.

The next result shows that the empirical centroids asymptotically concentrate around their mean-fields.
%\textcolor{red}{This result also pertains only to d=3, please generalize} \textcolor{olive}{Done.}
\begin{lemma}\label{lma:Hconcentrate}
    Assume that $G$ is sampled from a $d$-hSBM$(n,p,q)$ model, for some constant $d\geq 3$. Let the LPs of the tensor RWoHs be $x_{v;t}^{(k)}$ and assume that $\frac{nq^2}{\log n}\rightarrow \infty$.
    For sufficiently large $n$, any constant $\epsilon>0$ and a bounded constant $k$,
     \begin{align*}
     &\frac{2}{n}\left (\sum_{v\in V_0}x_{v;t}^{(k)} - \sum_{v\in V_1}x_{v;t}^{(k)} \right )\\
     &\in [(1-\epsilon)(\bar{a}^{(k)}-\bar{b}^{(k)}),(1+\epsilon)(\bar{a}^{(k)}-\bar{b}^{(k)})],
    \end{align*}
    with probability at least $1-o(1)$.
\end{lemma}
The proof of Lemma~\ref{lma:Hconcentrate} is deferred to Supplement~\ref{app:pflma2}. The above results show that if one sets $\gamma_k = \bar{w}_k$ in the GPR formulation, it will asymptotically approach the geometric discriminant function or a tight approximation thereof. Independent on the choice of the parameter $\alpha$, the geometric discriminant function of PPR does not match that of the tensor RWoHs on $d$-hSBM. Only the choice of $\gamma_k$ suggested by our analysis will allow the aforementioned result to hold true.

%\textcolor{red}{The CE proof is not that interesting, maybe you can move it to the Appendix?} \textcolor{olive}{Agree.}
\vspace{-0.1cm}
\section{Proofs}\label{sec:T-RWoH}
\vspace{-0.1cm}
Due to space limitations, we relegate all relevant proofs pertaining to the clique-expansion method to the Supplement and exclusively focus on the tensor case.

The main technical difficulty associated with tensor RWoHs is the large size of the state space, equal to $2^{d-1}$, which makes an analysis akin to the one described in Theorem~\ref{MAINTHM:CE} difficult. The main finding of this section is that for the $d$-hSBM, the centroid distances are governed by small recurrences involving only $d-1$ states. The proof supporting this observation comprises two steps, the first step of which is similar to the proof of Theorem~\ref{MAINTHM:CE}. The second step of the proof describes how to reduce the state space of the recurrence. For simplicity, we start with $d=3$ and then generalize the analysis for arbitrary $d$. For notational convenience, we write $\mathbf{Y}^{(k)} = [Y_1^{(k)},Y_2^{(k)},Y_3^{(k)},Y_4^{(k)}]^T$.

%\textcolor{blue}{The following theorem overlaps with theorem 3 and can be removed. This part only contains proof of theorem 3.}
%First we show that $\bar{y}_{i,j;t}^{(k)}$ can be represented by a $2^{3-1} = 4$ states of recurrence relations, which can be viewed as a variant of Theorem~\ref{MAINTHM:CE} for the tensor RWoH case.
\begin{theorem}\label{thm:post123456}
    Let $G$ be sampled from $3$-hSBM$(n,p,q)$ and let the tensor RWoHs be associated with an initial vector $y^{(0)}_{s_1,s_2; t} = 1$ and $y^{(0)}_{v_1,v_2;t} = 0,\;\forall (v_1,v_2)\neq (s_1,s_2)$, where $s_1$ and $s_2$ are selected independently and uniformly at random from $V_0$. Furthermore, let $\bar{y}^{(0)}_{t} = \E y^{(0)}_{t}$. Then for all $k\geq 0$
    \begin{align}
        \bar{y}_{i,j;t}^{(k)} =
        \begin{cases}
            % Y_1^{(c)} & \text{if } (i,j) \in S\times S\\
            % Y_2^{(c)} & \text{if } (i,j) \in S\times V_a\\
            % Y_3^{(c)} & \text{if } (i,j) \in S\times V_b\\
            % Y_4^{(c)} & \text{if } (i,j) \in V_a\times S\\
            Y_1^{(k)} & \text{if } (i,j) \in V_0\times V_0,\\
            Y_2^{(k)} & \text{if } (i,j) \in V_0\times V_1,\\
            % Y_7^{(c)} & \text{if } (i,j) \in V_b\times S\\
            Y_3^{(k)} & \text{if } (i,j) \in V_1\times V_0,\\
            Y_4^{(k)} & \text{if } (i,j) \in V_1\times V_1,\\
            % 0 & \text{otherwise}
        \end{cases}
    \end{align}
where $\mathbf{Y}^{(0)} = \frac{4}{n^2}[1,0,0,0]^T$ and
    \begin{align*}%\label{eq:MFA2MCposteq1}
        &\begin{bmatrix}
            Y_1^{(k+1)}\\
            Y_2^{(k+1)}\\
            Y_3^{(k+1)}\\
            Y_4^{(k+1)}\\
        \end{bmatrix}
        = %\frac{1}{N^{(k)}}
        \begin{bmatrix}
        \frac{np}{2} & 0 & \frac{nq}{2} & 0\\
        \frac{nq}{2} & 0 & \frac{nq}{2} & 0\\
        0 & \frac{nq}{2} & 0 & \frac{nq}{2}\\
        0 & \frac{nq}{2} & 0 & \frac{np}{2}
        \end{bmatrix}
        \begin{bmatrix}
            Y_1^{(k)}\\
            Y_2^{(k)}\\
            Y_3^{(k)}\\
            Y_4^{(k)}
        \end{bmatrix}.
        % &
        % \begin{bmatrix}
        %     Y_1^{(0)}\\
        %     Y_2^{(0)}\\
        %     Y_3^{(0)}\\
        %     Y_4^{(0)}
        % \end{bmatrix}
        % = \frac{4}{n^2}
        % \begin{bmatrix}
        %     1\\
        %     0\\
        %     0\\
        %     0
        % \end{bmatrix}.\nonumber
    \end{align*}
\end{theorem}
%Note that $N^{(c)}$ is nothing but the normalization for the step $c$.
\begin{proof} %\textcolor{blue}{Please first state the high-level logic and then go in each part.}
The proof proceeds by induction: The base case $k = 0$ is clearly true. For the induction step, assume that the hypothesis holds for $1,2,...,k$. Then
    \begin{align*}
      & \forall i,j\in V_0, \; Y_1^{(k+1)} = \bar{y}_{i,j;t}^{(k+1)} = \sum_{l=1}^{n}\E A_{lij}\bar{y}_{l,i;t}^{(k)} \\
      &= \sum_{l\in V_0}\E A_{lij}\bar{y}_{l,i;t}^{(k)}+\sum_{l\in V_1} \E A_{lij}\bar{y}_{l,i;t}^{(k)} \\
      & = \sum_{l\in V_0}\E A_{lij}Y_1^{(k)}+\sum_{l\in V_1}\E A_{lij}Y_3^{(k)} = \frac{np}{2}Y_1^{(k)}+\frac{nq}{2}Y_3^{(k)}.
    \end{align*}
Similar expressions may be derived for $Y_2^{(k+1)},Y_3^{(k+1)},Y_4^{(k+1)}$. This completes the proof.
    %Hence we have proved Theorem~\ref{thm:post123456}.
    %\small
    %\begin{align*}
    %    &\sum_{i,j\in[n]}D_{ij}\bar{y}_{ij}^{(c)} \\
    %    &= \frac{n^2}{4}(\frac{n(p+q)}{2}(Y_1^{(c)}+Y_4^{(c)}) + nq(Y_2^{(c)} + Y_3^{(c)})) = N^{(c)}
    %\end{align*}
    %\normalsize
    %So the recurrence relation of mean field post-normalized random walk becomes
    %\begin{align*}
    %    \bar{y}^{(c+1)}_{ij} = \frac{1}{N^{(c)}}\sum_{l = 1}^{n}\E A_{lij}\bar{y}_{li}^{(c)}
    %\end{align*}
    %Next we compute $\sum_{l = 1}^{n}\E A_{lij}\bar{y}_{li}^{(c)}$ directly.
    %\begin{align*}
    %    & \sum_{l = 1}^{n}\E A_{lij}\bar{y}_{li}^{(c)} = \sum_{l \in V_1}\E A_{lij}\bar{y}_{li}^{(c)} + \sum_{l \in V_2}\E A_{lij}\bar{y}_{li}^{(c)} \\
     %   &=
      %  \begin{cases}
      %       \frac{np}{2}Y_1^{(c)} + \frac{nq}{2}Y_3^{(c)} & \text{ if } (i,j)\in V_1\times V_1\\
      %       \frac{nq}{2}Y_1^{(c)} + \frac{nq}{2}Y_3^{(c)} & \text{ if } (i,j)\in V_1\times V_2\\
       %      \frac{nq}{2}Y_2^{(c)} + \frac{nq}{2}Y_4^{(c)} & \text{ if } (i,j)\in V_2\times V_1\\
      %       \frac{nq}{2}Y_2^{(c)} + \frac{np}{2}Y_4^{(c)} & \text{ if } (i,j)\in V_2\times V_2
    %    \end{cases}
    %\end{align*}
    %This exactly implies~\eqref{eq:MFA2MCposteq1}. Together we complete the proof.
\end{proof}
Next we show how to reduce the number of states to $d-1$. To this end, we simplify $\bar{x}_{i;t}$ as
\begin{align*}
  & \bar{x}_{j;h}^{(k)} = \frac{ \sum_{i = 1}^{n}\bar{y}_{i,j;h}^{(k)}}{ \sum_{i,l = 1}^{n}\bar{y}_{i,l;h}^{(k)}} =
  \begin{cases}
        \bar{a}^{(k)} = \frac{2}{n}\frac{Y_1^{(k)} + Y_3^{(k)}}{\sum_{m=1}^{4}Y_m^{(k)}}, & \mbox{if } j\in V_0 \\
        \bar{b}^{(k)} = \frac{2}{n}\frac{Y_2^{(k)} + Y_4^{(k)}}{\sum_{m=1}^{4}Y_m^{(k)}}, & \mbox{if } j\in V_1.
      \end{cases}
\end{align*}
The centroid distance $\bar{w}= \bar{a} - \bar{b}$ may be written as
\begin{align*}
  & \bar{w}_k= \bar{a}^{(k)} - \bar{b}^{(k)} = \frac{2}{n}\frac{\beta_1(k)}{\zeta_1(k)},\text{ where}\\
  & \beta_1(k) = [1,-1,1,-1]\mathbf{Y}^{(k)},\,
                            \zeta_1(k) = [1,1,1,1]\mathbf{Y}^{(k)}.
\end{align*}
% \begin{align*}

% \end{align*}
We introduce next the following auxiliary variables:
\begin{align*}
  & \beta_2(k) = [1,0,0,-1]\mathbf{Y}^{(k)},\,
                            \zeta_2(k) = [1,0,0,1]\mathbf{Y}^{(k)}.
\end{align*}
In the expression for $\beta_1(k)$, we replace $\mathbf{Y}^{(k)}$ by $\mathbf{Y}^{(k-1)}$ by invoking the recurrence of Theorem~\ref{thm:post123456}. One can then show that the recurrence for $\mathbf{Y}^{(k)}$ may be replaced by a recurrence for $\beta_1(k)$ and $\beta_2(k)$:
%The same approach  $\beta_2(k)$, resulting in the following recurrence relation:
% we have the follow
% Next note that the transition matrix for $Y$ in~\eqref{eq:MFA2MCposteq1} can be decompose as $\frac{n}{2}(H+E)$ where
% \begin{align*}
%     \mathbf{H} \eqDef \begin{bmatrix}
%         q\quad & 0\quad & q\quad & 0\\
%         q\quad & 0\quad & q\quad & 0\\
%         0\quad & q\quad & 0\quad & q\\
%         0\quad & q\quad & 0\quad & q
%     \end{bmatrix},\;\mathbf{E} \eqDef\begin{bmatrix}
%         p-q\quad & 0\quad & 0\quad & 0\\
%         0\quad & 0\quad & 0\quad & 0\\
%         0\quad & 0\quad & 0\quad & 0\\
%         0\quad & 0\quad & 0\quad & p-q
%     \end{bmatrix}.
% \end{align*}
% Then we have the following key observation
% \small
% \begin{align*}
%     &[1,-1,1,-1]\mathbf{H} = 0\times [1,-1,1,-1],\;[1,-1,1,-1]\mathbf{E} = (p-q)[1,0,0,-1]\\
%     &[1,0,0,-1]\mathbf{H} = q\times [1,-1,1,-1],\;[1,0,0,-1]\mathbf{E} = (p-q)[1,0,0,-1].
% \end{align*}
% \normalsize
% Now if we let $\beta_2(k) \eqDef [1,0,0,-1]Y$, then we can have the following recurrence relation
\begin{align}
    \begin{bmatrix}
        \beta_1(k+1)\\
        \beta_2(k+1)
    \end{bmatrix}
    = \frac{n}{2}
    \begin{bmatrix}
        0 & p-q\\
        q & p-q
    \end{bmatrix}
    \begin{bmatrix}
        \beta_1(k)\\
        \beta_2(k)
    \end{bmatrix}.
\end{align}
For $\zeta$, one can derive the following similar result:
\begin{align}
    \begin{bmatrix}
        \zeta_1(k+1)\\
        \zeta_2(k+1)
    \end{bmatrix}
    = \frac{n}{2}
    \begin{bmatrix}
        2q & p-q\\
        q & p-q
    \end{bmatrix}
    \begin{bmatrix}
        \zeta_1(k)\\
        \zeta_2(k)
    \end{bmatrix}.
\end{align}
% \textcolor{red}{(I don't know how to write it in a compact way)
% The proof is identical to the case of $d=3$. Then similarly we can decompose the matrix into $\mathbf{H} + \mathbf{E}$. Then define $\beta_1(k) = [1,-1,...,1,-1]Y^{(k)}$, $\beta_{d-1}(k) = [1,0,...,0,-1]Y^{(k)}$, $\beta_{d-2}(k) = [1,0,...,-1,1,0,...,-1]Y^{(k)}$... and so on. Similar to the case of $\zeta$. Then in general we have}
This approach generalizes for arbitrary $d$, but we defer the detailed analysis to Supplement~\ref{app:remainpfthm2}. Let $e_{i} = [1,0,...,0,-1]$ where there the runlength of zeros equals $2^{i}-2$.
Then, $\beta_i(k) = [e_i,e_i,...,e_i]Y^{(k)}$; a similar expression is valid for $\zeta$ with all $-1$s changed to $1$s. As for the case $d=3$, one can establish the following recurrence relations:
%It is not hard to see that in general we have \textcolor{blue}{I suggest to have a more detailed description here about which steps can be generalized in which way. Because things are not so obvious. }
\small
\begin{equation}\label{margsimp7}
	\begin{bmatrix}
		\beta_1(k)\\
        \vdots\\
        \beta_{d-1}(k)
	\end{bmatrix}
     = \frac{n}{2}
     \begin{bmatrix}
     	0 & \cdots & 0 & 0 & p-q\\
        q & 0 & \cdots & 0 & p-q\\
        0 & \ddots & \vdots & 0 & p-q\\
        0 & \cdots & q & 0 & p-q\\
        0 & \cdots & 0 & q & p-q\\
     \end{bmatrix}
     \begin{bmatrix}
		\beta_1(k-1)\\
        \vdots\\
        \beta_{d-1}(k-1)
	\end{bmatrix},
\end{equation}
\normalsize
and
\small
\begin{equation}\label{margsimp8}
	\begin{bmatrix}
		\zeta_1(k)\\
        \vdots\\
        \zeta_{d-1}(k)
	\end{bmatrix}
     = \frac{n}{2}
     \begin{bmatrix}
     	2q & \cdots & 0 & 0 & p-q\\
        q & 0 & \cdots & 0 & p-q\\
        0 & \ddots & \vdots & 0 & p-q\\
        0 & \cdots & q & 0 & p-q\\
        0 & \cdots & 0 & q & p-q\\
     \end{bmatrix}
     \begin{bmatrix}
		\zeta_1(c-1)\\
        \vdots\\
        \zeta_{d-1}(c-1)
	\end{bmatrix}.
\end{equation}
\normalsize
This complete the proof of Theorem~\ref{thm:abapprox}.
%Hence we complete the proof.
\vspace{-0.1cm}
\section{Construction of GPR based on landing probabilities}\label{sec:CoGPR}
\vspace{-0.1cm}
%\textcolor{blue}{Too sketchy...  Please give more formal description of the setting-ups. Again, assume you are telling a story so be careful about the transition. Also give a more formal introduction of Fisher. Why and when it can be important? And give an suggestion on this as a future work? How do you partition the graphs or how do you choose theoreshold? Do you consider the variance of the results?  Please consider some statements to make this section more exciting (at least be very serious and formal), although it is not an important part...  }
In what follows, we use the results of our theoretical results to propose new GPR methods for hypergraph clustering. Following~\cite{kloumann2017block}, the geometric discriminant of interest equals $w^Tx_v,$ where $x_v$ is the landing probability vector of the vertex $v$. If only the first moments of the LPs are available, the optimal choice of $w$ corresponding to the maximal marginal separator of the centroids is given in Theorems~\ref{MAINTHM:CE} and~\ref{thm:abapprox} for clique-expansion RWoHs and tensor RWoHs, respectively.

The geometric discriminant only takes the first-order moments of LPs into account. As pointed out in~\cite{kloumann2017block}, the Fisher discriminant is expected to have better classification performance since it also make use of the covariances of LPs (see Figure~\ref{fig:explain}). More precisely, the Fisher discriminant takes the form $\left ( \Sigma^{-1}w\right)^Tx_v,$ where $x_v$ is the landing probability vector of the vertex $v$ and $\Sigma$ is the covariance matrix of the landing probability vector. The authors of~\cite{kloumann2017block} \emph{empirically} leveraged the information about the second-order moments of LPs. They showed that the Fisher discriminant has a performance that nearly matches that of belief propagation, the statistically optimal method for community detection on SBM~\cite{abbe2015detection,zhang2014scalable,mossel2014belief}. We therefore turn our attention to Fisher discriminant corresponding to clique-expansion and tensor RWoHs.

\begin{figure}[!htb]
\centering
  \subfigure[\scriptsize{Geometric discriminant.}\label{fig:geodis_plot}]{\includegraphics[width=0.48\linewidth]{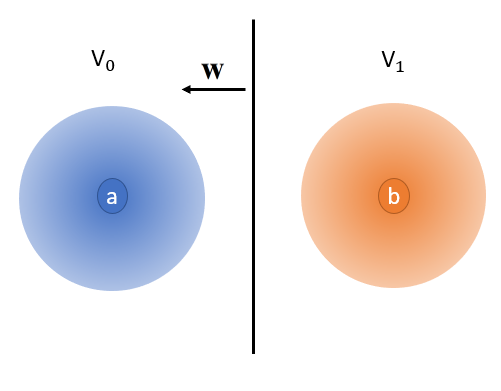}}
  \subfigure[\scriptsize{Fisher discriminant.}\label{fig:Fishdis_plot}]{\includegraphics[width=0.48\linewidth]{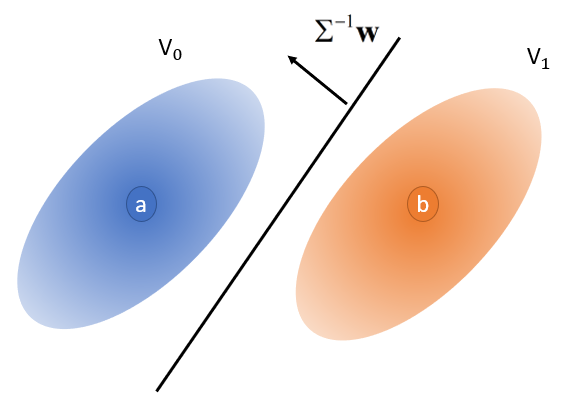}}
  \vspace{-0.5cm}
  \caption{Illustration of geometric and Fisher discriminants. Consecutive step LPs are correlated as the random walks have memory and one needs to take the covariance into account. The gradations in the colors reflect the density of the LPs in the ambient space. }\label{fig:explain}
  \vspace{-0.4cm}
\end{figure}
%However, the consecutive steps of LP are correlated due to the nature of random walk which is also pointed out in~\cite{kloumann2017block} (e.g. two-hop LP from a given seed correlate with the three-hop LP). Hence we need to take the second order moment (covariance) into account in order to achieve a good clustering performance in general cases.
%Moreover, the performance of clustering based on LPs depends on both centroid distance and the covariance which can be easily understood from Figure~\ref{fig:explain}.

Recall that our theoretical results shows that tensor RWoHs lead to larger centroid distances compared to those of clique-expansion RWoHs. Most importantly, the difference between the centroid distances of the two methods increase with the hyperedge size $d$. Hence, for large hyperedges sizes the theoretical results suggest that one should not directly use clique-expansion combined with PR methods. On the other hand, clique-expansion leads to reductions in the variance of random walks. This is intuitively clear since entries of the clique-expanded adjacency matrix contain sums of entries of the original adjacency tensor; hence the adjacency matrix obtained through clique-expansion will be ``closer'' to its expectation, implying a smaller variance. This also follows from Lemma~\ref{lma:concentrate} and ~\ref{lma:Hconcentrate} by observing that the empirical centroids of clique-expansion RWoHs converge faster as $n$ grows. This points to an important bias-variance trade-off between clique-expansion and tensor RWoHs. We therefore propose the following hybrid random walk scheme combining clique-expansion and tensor methods, referred to ``CET RWoHs''. The gist of the CET approach is not to replace a hyperedge by a complete graph as is done in clique-expansion, but replace it by a complete lower order hypergraph instead. On the reduced order hypergraph one can then apply the tensor RWoHs both to increase the centroid distance and to ensure smaller computational and space complexity.
\vspace{-0.1cm}
\section{Simulations}
\vspace{-0.1cm}
In the examples that follow, all results are obtained by averaging over $20$ independent trials.

The first test illustrates the clustering performance of geometric and Fisher discriminant corresponding to clique-expansion and tensor RWoHs for $3$-hSBM$(100,p,q)$ with a uniform seed initialization. More precisely, we start with $\bar{y}_{ce}^{(0)}$ and $\bar{y}_{t}^{(0)}$, respectively. Subsequently, we use
$k=6$ steps of the random walk for both the clique-expansion RWoHs and tensor RWoHs; our choice is governed by the fact that the centroid distance of the clique-expansion LPs with $k=6$ is close to $0$. Figure~\ref{fig:expResult_old} shows that for both geometric and Fisher discriminant, using the LPs of tensor RWoHs results in better clustering performance compared to that of clique-expansion RWoHs. This supports our theoretical results.% presented in the previous sections.
% that tensor based method has a larger centroid distance which leads to a better clustering performance with uniform initialization.

\begin{figure}[!htb]
\centering
%   \subfigure[$4$-hyperedge]{\includegraphics[width=0.32\linewidth]{4-hyperedge.PNG}}
  \subfigure[\scriptsize{Geometric discriminant.}\label{fig:geoperfomance_plot}]{\includegraphics[width=0.48\linewidth]{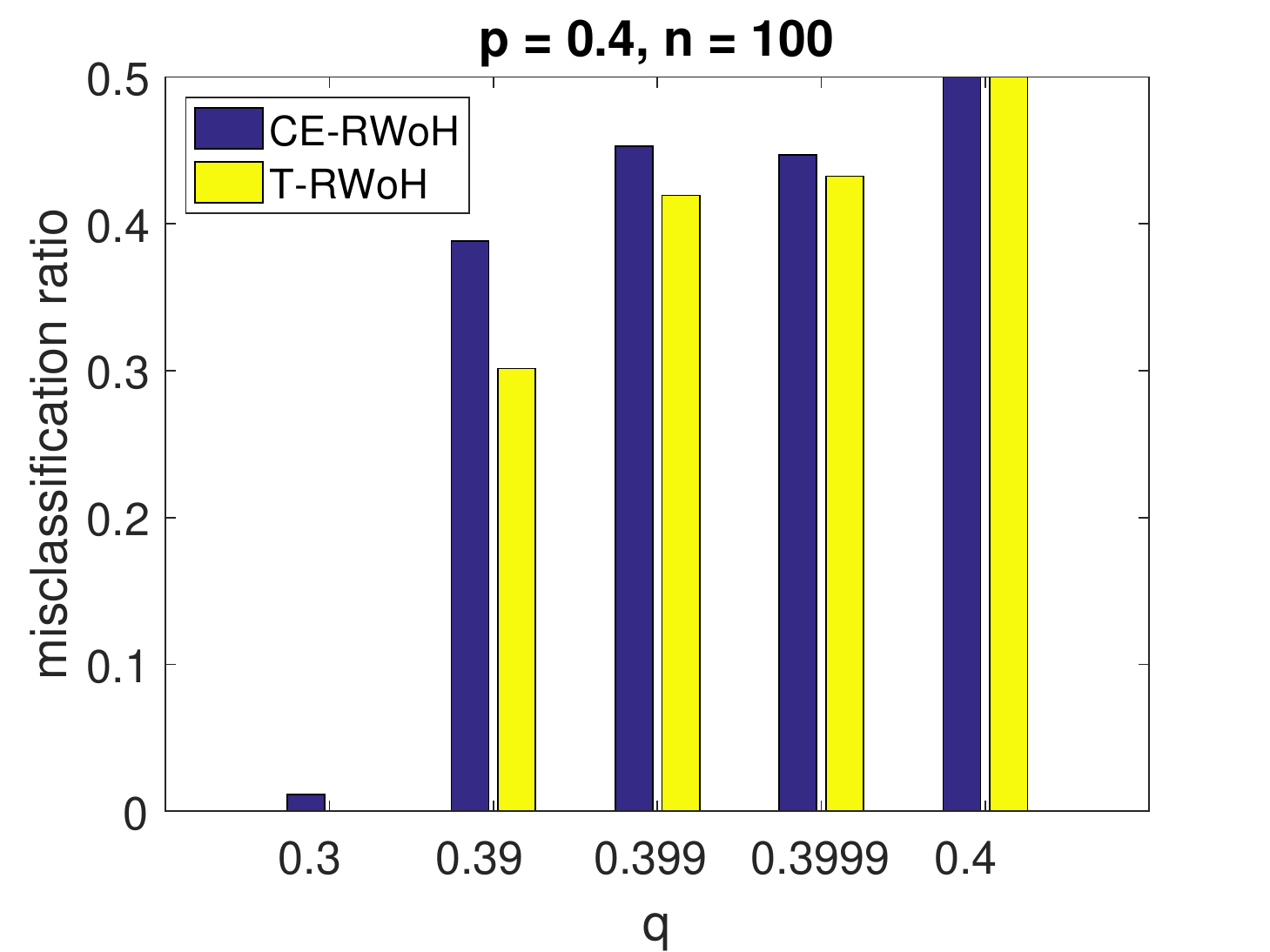}}
  \subfigure[\scriptsize{Fisher discriminant.}\label{fig:fisherperfomance_plot}]{\includegraphics[width=0.48\linewidth]{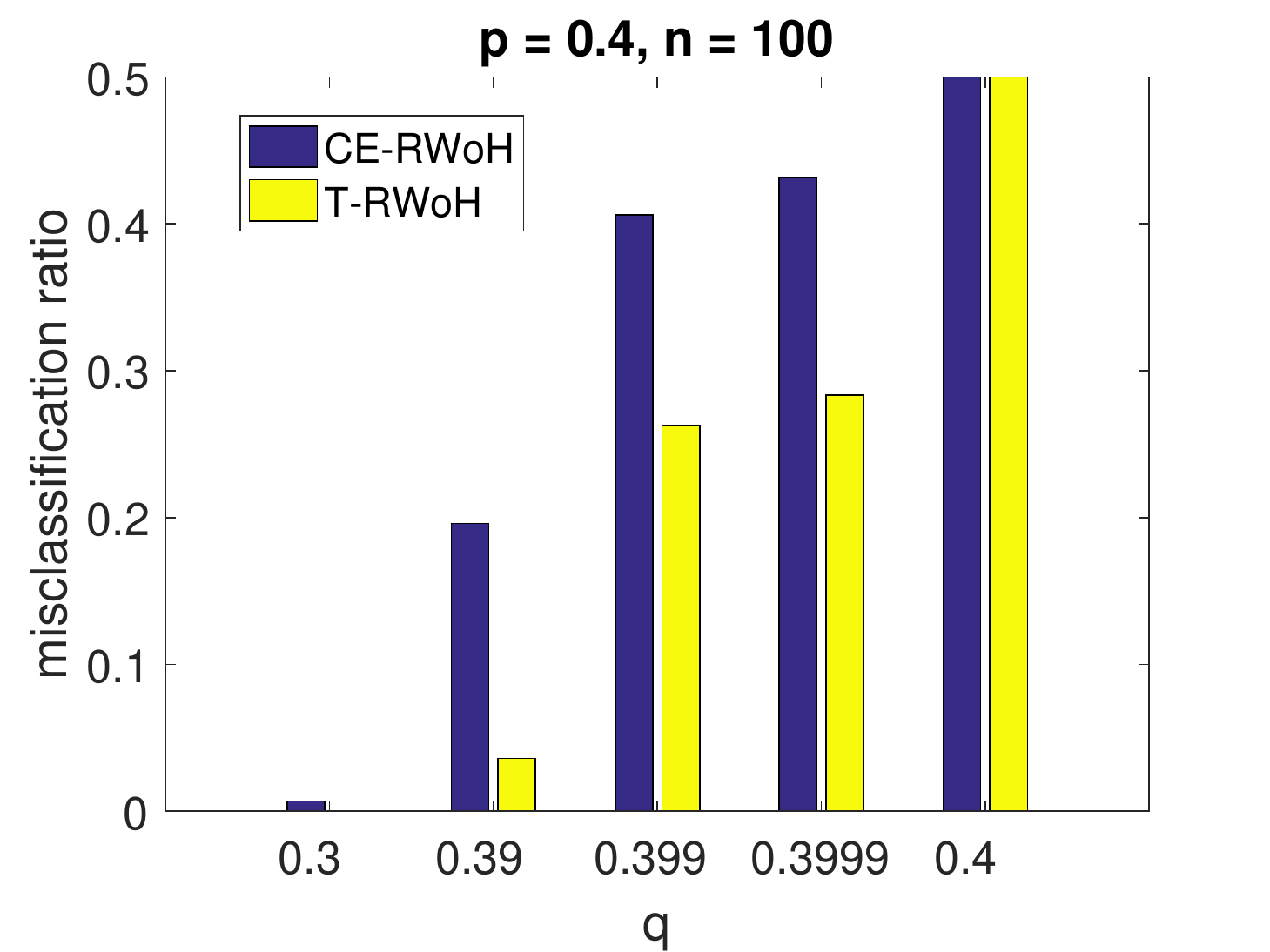}}
  \vspace{-0.5cm}
  \caption{Clustering performance of the CE (clique-expansion) and T (tensor) methods with uniform initialization on a $3$-hSBM.}\label{fig:expResult_old}
  \vspace{-0.2cm}
\end{figure}

However, in practice, one rarely uses a uniform initialization as it implies (partial) prior knowledge of the cluster structure: Seed-set expansion is usually of interest in applications where the seeds are user-defined. To illustrate the performance of the clique-expansion, tensor and CET methods in this setting, we also used single-vertex-seed initializations $y_{ce}^{(0)}$ and $y_{t}^{(0)}$, respectively. Figure~\ref{fig:expResult} provides simulations for a $4$-hSBM$(100,p,q)$, demonstrating that when only the first moment is used (i.e., when the discriminant is geometric), clique-expansion RWoHs offer the best performance.
This may be explained by observing that the LPs are correlated and clique-expansion RWoHs has a smaller variance than the other methods. On the other hand, if we additionally use the second moment (i.e., when the discriminant is Fisher), then the CET RWoHs has the best performance in almost all parameter regimes while the tensor RWoHs has the best performance only when $p-q$ is close to $0$. This finding matches our results and their interpretation in Section~\ref{sec:CoGPR}, indicating that CET RWoHs offers good bias-variance trade-offs. As the difference of the centroid distances of clique-expansion and tensor RWoHs grows as the hyperedge size $d$ increases, the performance gain of CET RWoHs is expected to be even larger for higher order hypergraphs. It remains an open question how to choose the best combination of hypergraph projections and tensor RWoHs with respect to both performance and computational complexity.

\begin{figure}[!htb]
\centering
%   \subfigure[$4$-hyperedge]{\includegraphics[width=0.32\linewidth]{4-hyperedge.PNG}}
  \subfigure[]{\includegraphics[width=0.49\linewidth]{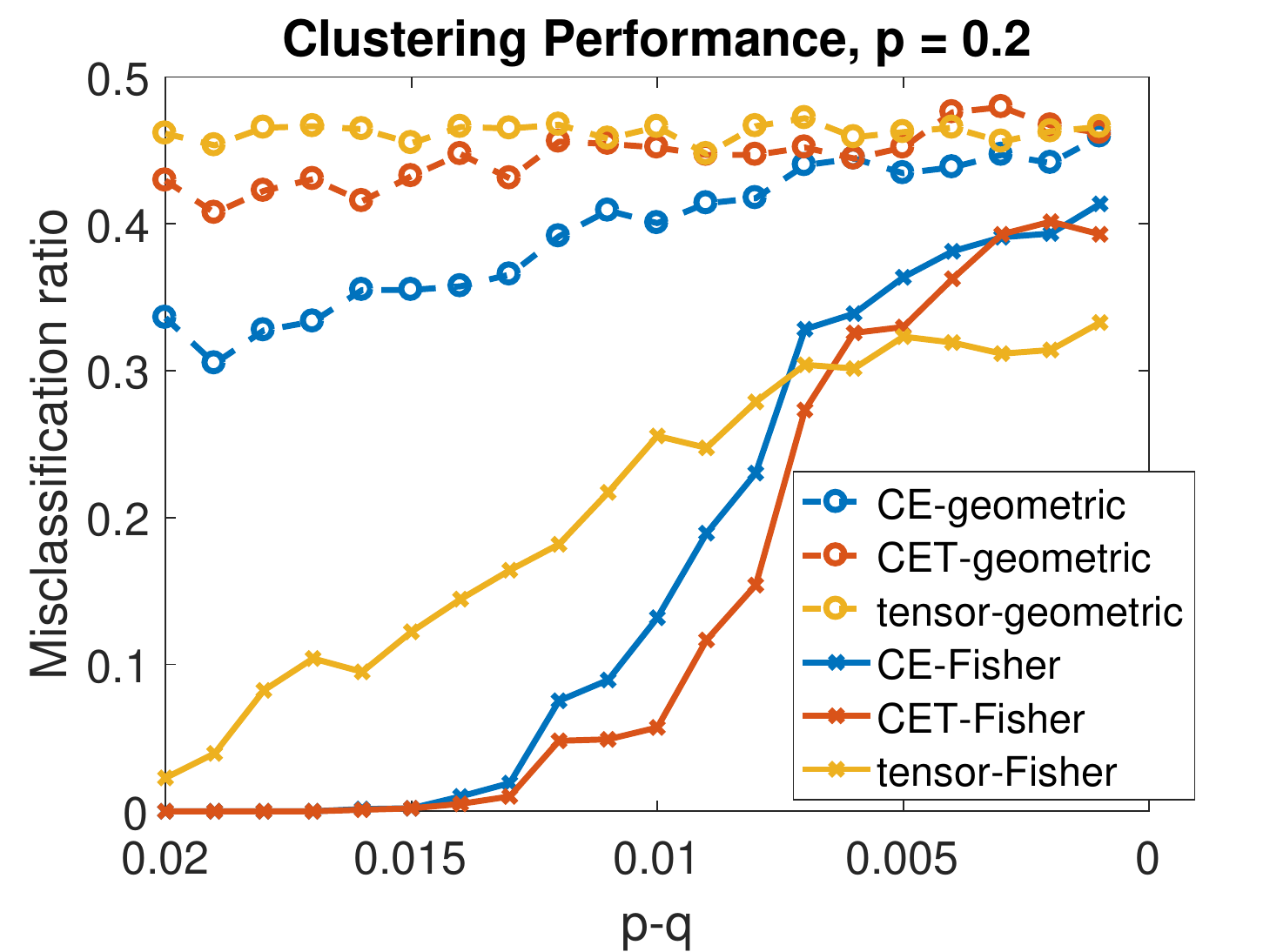}}
%   \subfigure[]{\includegraphics[width=0.33\linewidth]{p4_NEW.pdf}}
  \subfigure[]{\includegraphics[width=0.49\linewidth]{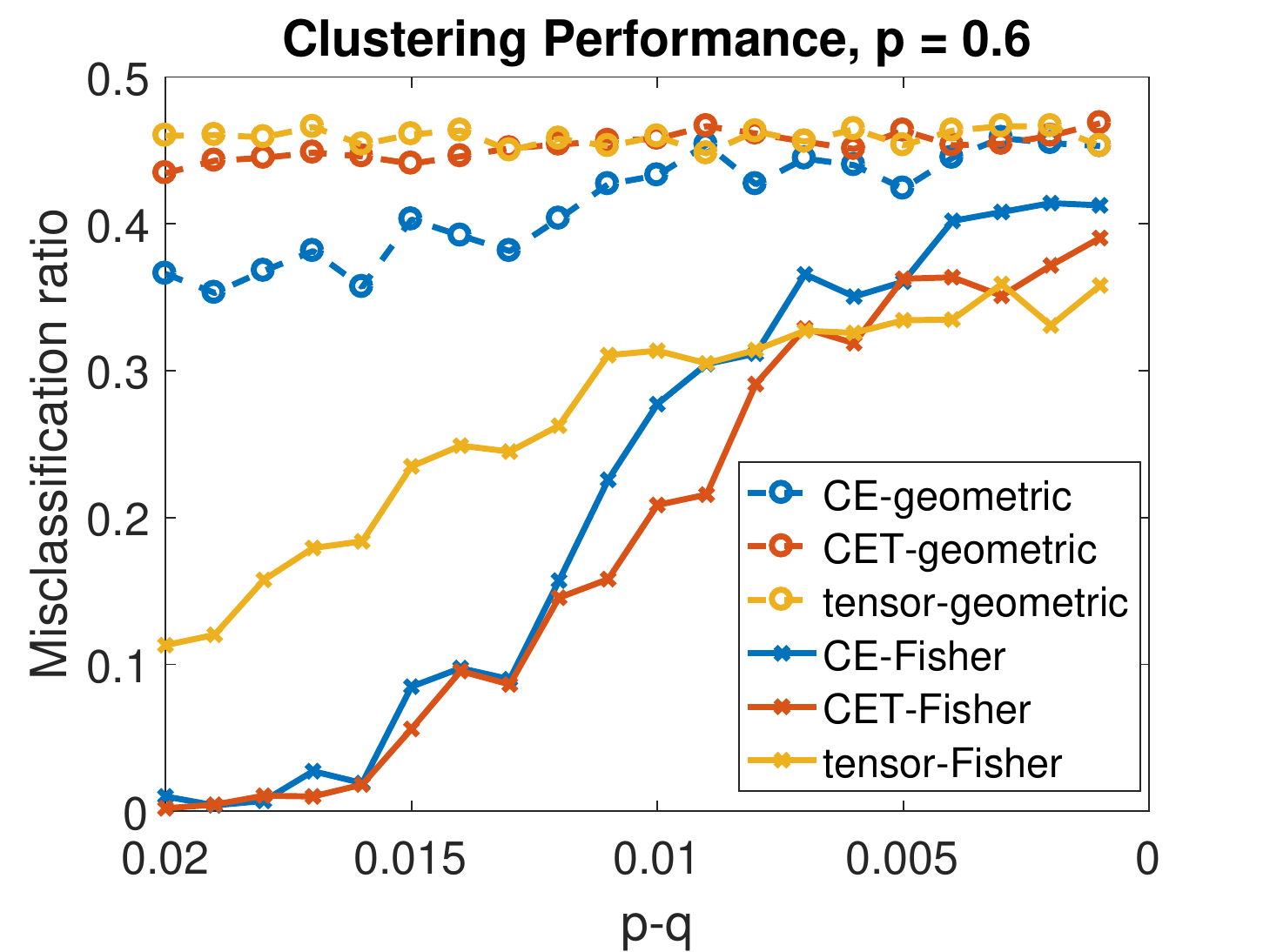}}
  \vspace{-0.5cm}
  \caption{Clustering performance of the CE (clique-expansion), tensor and CET methods with single-seed-vertex initialization on $4$-hSBM.}\label{fig:expResult}
\end{figure}
\textbf{Acknowledgment:} The work was supported by the NSF grant 1956384 and the NSF Center for Science of Information (CSoI) housed at Purdue University.

%\subsubsection*{Acknowledgements}

%Use unnumbered third level headings for the acknowledgements title.
%All acknowledgements go at the end of the paper.

%\subsubsection*{References}
\bibliographystyle{IEEEtran}
\bibliography{example_paper}

\newpage
\appendix

\section{Clique-expansion RWoHs: Proofs}\label{sec:CEhRW} %\textcolor{blue}{I think most part of this section can be removed. Merge some of the results into Section 3, as I suggested in the following blue text. This section contains only the way to derive the recursion.}
In what follows, we provide a proof of Theorem~\ref{MAINTHM:CE}.% and first turn our attention to $\bar{x}_{ce}^{(1)}$.

Recall that $y_{i;ce}^{(k)}$ counts the number of paths of length $k$ starting from some seed vertex $s$ in $G^{(ce)}$ and that $\bar{y}_{i;ce}^{(k)}$ denotes its mean-field complement. Then
\begin{align*}
     \forall i\in V_0,\;\bar{y}^{(1)}_{i;ce} &=  \frac{2}{n}\sum_{s\in V_0}\sum_{v_1,...,v_{d-2}}\E A_{s,i,v_1,...,v_{d-2}} \\
    & = \left(\frac{n}{2}\right)^{d-2}(p+(2^{d-2}-1)q); \\
     \forall i\in V_1,\;\bar{y}^{(1)}_{i;ce} &=  \frac{2}{n}\sum_{s\in V_0}\sum_{v_1,...,v_{d-2}}\E A_{s,i,v_1,...,v_{d-2}} = n^{d-2}q.
    %& \bar{x}^{(1)}_s = e_1^T\bar{\mathbf{W}}_{ce}x^{(0)} \cong \frac{1}{\bar{F}}\sum_{l}\E A_{1sl} = 0
\end{align*}
Let  $\bar{F} = \sum_{i\in V}\bar{y}^{(1)}_{i;ce} = \frac{n}{2}((\frac{n}{2})^{d-2}(p+(2^{d-2}-1)q)+n^{d-2}q) = (\frac{n}{2})^{d-1}(p+(2^{d-1}-1)q)$. By definition
\begin{align*}
   \forall i\in V_0,\;\bar{x}^{(1)}_{i;ce} &= \frac{1}{\bar{F}}(\frac{n}{2})^{d-2}(p+(2^{d-2}-1)q) \\
   &= \frac{2}{n}\frac{p+(2^{d-2}-1)q}{p+(2^{d-1}-1)q},\\
   \forall i\in V_1,\;\bar{x}^{(1)}_{i;ce} &= \frac{n^{d-2}q}{\bar{F}} = \frac{2}{n}\frac{2^{d-2}q}{p+(2^{d-1}-1)q}.
\end{align*}
Based on this one-step analysis and the low-rank structure of the model one may conjecture that for general $k$, $\bar{x}^{(k)}_{v;ce}$ only takes two distinct values depending on $v$. %\textcolor{olive}{This is simply due to the low-rank structure of SBM.}
We prove this intuitive observation by induction.
%\begin{theorem}\label{thm:MFA_CE}
%Let $G$ be generated from $3$-hSBM$(n,p,q)$ and $G^{(ce)}$ be the graph obtained from $G$ by clique-expansion (See Section ~\ref{sec:RWoH}). The CE-RWoH are associated with a initial vector $y^{(0)}_{s; ce} = 1$ and $y^{(0)}_{v;ce} = 0,\;\forall v\neq s$, where $s$ is uniformly at random chosen from $V_0$. Let $\bar{y}^{(0)}_{ce} = \E y^{(0)}_{ce}$.  Then for all $k\geq 0$ we have
%    \begin{align}
%    \bar{x}^{(k)}_{i;ce} =
%    \begin{cases}
%       \bar{b}^{(k)}  & \text{if }i\in V_1
%     \end{cases}
% \end{align}
% where \textcolor{blue}{I think we can also put the following recursion into Thm 1, as that in Thm 3. Then the proof of this section only focuses on how to derive this recursion. Actually, as only as we have this recursion, we even do not need to show something like the eigenvalue decomposation, which is too normal to be shown in a research paper. Check our KDD paper, what we have there are only the recursion and the final results.}
% \begin{align}
%     &\begin{bmatrix}
%     \bar{a}^{(k)}\\
%     \bar{b}^{(k)}
%     \end{bmatrix}
%     = \frac{1}{\bar{F}}
%     \begin{bmatrix}
%         \frac{n^2}{4}(p+q)  &   \frac{n^2}{2}q\\
%         \frac{n^2}{2}q  &   \frac{n^2}{4}(p+q)
%     \end{bmatrix}
%     \begin{bmatrix}
%     \bar{a}^{(k-1)}\\
%     \bar{b}^{(k-1)}
%     \end{bmatrix}\\
%     &\begin{bmatrix}
%     \bar{a}^{(0)}\\
%     \bar{b}^{(0)}
%     \end{bmatrix}
%     = \frac{2}{n}
%     \begin{bmatrix}
%     1\\
%     0
%     \end{bmatrix}\label{eq:CE1}
% \end{align}
% \end{theorem}
Clearly, $\frac{n}{2}[\bar{a}^{(k)}\;\; \bar{b}^{(k)}]^T$ is a probability vector and the base case $k=0$ follows directly. For the induction step, assume that the induction hypothesis is true for $1,...,k$. Let $l = (v_1,...,v_{d-2})$ and write $l\in V$ to indicate that $v_1,...,v_{d-2}\in V$. Then
    \begin{align*}
        & \forall i\in V_0,\; \bar{x}^{(k+1)}_{i;ce} = \frac{\bar{y}_{i;ce}^{(k+1)}}{\sum_{j\in V}\bar{y}_{j;ce}^{(k+1)}} =  \frac{\sum_{j,l\in V}\E A_{jil}\bar{y}_{j;ce}^{(k)}}{\sum_{j,l,v\in V}\E A_{jvl}\bar{y}_{j;ce}^{(k)}} \\
        & = \frac{\sum_{j,l\in V}\E A_{jil}\bar{x}_{j;ce}^{(k)}}{\sum_{j,l,v\in V}\E A_{jvl}\bar{x}_{j;ce}^{(k)}} \\
        &= \frac{\sum_{j\in V_0}\sum_{l}\E A_{jil}\bar{a}^{(k)} + \sum_{j\in V_1}\sum_{l}\E A_{jil}\bar{b}^{(k)} }{\sum_{j\in V_0}\sum_{l,v}\E A_{jvl}\bar{a}^{(k)} + \sum_{j\in V_1}\sum_{l,v}\E A_{jvl}\bar{b}^{(k)} }\\
        & =\frac{(\frac{n}{2})^{d-1}(p+(2^{d-2}-1)q)\bar{a}^{(k)} + \frac{n^{d-1}q}{2}\bar{b}^{(k)}}{\bar{F}\frac{n}{2}(\bar{a}^{(k)}+\bar{b}^{(k)})} = \bar{a}^{(k+1)}.
    \end{align*}
A similar argument may be used to characterize $\bar{b}$.

\section{Theorem~\ref{thm:abapprox}: Large $d$ regime}\label{app:thm2larged}

We start by computing the eigenvalues of the update matrix for $\beta$. For this purpose, rewrite~\eqref{margsimp7} as follows:
\begin{equation}
    \frac{nq}{2}\times
    \begin{bmatrix}
     	0 & \cdots & 0 & 0 & R\\
        1 & 0 & \cdots & 0 & R\\
        0 & \ddots & \vdots & 0 & R\\
        0 & \cdots & 1 & 0 & R\\
        0 & \cdots & 0 & 1 & R\\
     \end{bmatrix} \triangleq \frac{nq}{2} \mathbf{Q_1}, R \triangleq \frac{p-q}{q}.
\end{equation}
Note that $\mathbf{Q_1}$ takes the form of a companion matrix of dimension $d-1$. It is well known that if $\mathbf{Q_1}$ has $d-1$ distinct eigenvalue, then it can be diagonalized as follows
\begin{equation*}
    UQ_1U^{-1} = Diag(\lambda_1,...,\lambda_{d-1}),
\end{equation*}
where $U$ is the Vandermonde matrix associate with $\lambda_i$, and $U_{ij} = \lambda_{i}^{j-i}$. Note that $\mathbf{Q_1}$ has full rank when $p>q$.

Next we characterize the eigenvalues $\lambda_i$ by writing down the characteristic polynomial of $\mathbf{Q_1}$:
\begin{equation}\label{eq:eigenbeta1}
    t^{d-1} = R(t^{d-2} + t^{d-3} + ... + 1).
\end{equation}
%\textcolor{blue}{The above equation can be further transformed into a recurrence sequence that only depends on $\beta_1$ and this is the key point that I can leverage to prove the general insight for arbitrary $d$. }
In general, there exists no closed form in terms of radicals when $d-1>5$. However, we can find the roots approximately by assuming that $\frac{p-q}{q}=R\neq \frac{1}{d-1}$ and using the following argument. Clearly, $t=1$ is not a root of the polynomial unless $R \triangleq \frac{p-q}{q} = \frac{1}{d-1}$, which we ruled out for the sake of simplifying the analysis. Since we allow $d$ to be large, $R = \frac{1}{d-1}$ implies $p$ will be close to $q$. Also, it is clear that $t = 0$ is not a root of the polynomial unless $R = 0$. Hence we will also assume that $t \neq 0$. Multiplying both sides of the polynomial expression~\eqref{eq:eigenbeta1} by $(t-1)$
we obtain
\begin{align}
    &t^{d}-(R+1)t^{d-1}+R = 0 \nonumber\\
    &\Leftrightarrow t + \frac{R}{t^{d-1}} = R+1. \label{eq:eigenbeta2}
\end{align}

The eigenvalues of the matrix under consideration satisfy either $|t|>1$ ,$|t|<1$ or $|t| = 1$. For $|t| > 1$, as $d$ is large, the LHS of \eqref{eq:eigenbeta2} will be close to $t$ which implies that there is a root close to $t = R+1>1$. For the case $|t| = 1$, we may write $t = e^{i\theta}$. Then $e^{i\theta}+ R e^{-i(d-1)\theta} = R+1$. Since $R\in \mathbb{R}$ we require $e^{i\theta} = e^{-i(d-1)\theta} = 1$ which violates our assumption that $t = 1$ is not a root.

On the other hand, when $|t|<1$, the LHS of~\eqref{eq:eigenbeta2} has a value close to $\frac{R}{t^{d-1}}$. Thus, the remaining $d-2$ eigenvalues will be close to the $d-1$ complex roots of $t^{d-1} = \frac{R}{R+1}$. However, note that the real root $t = (\frac{R}{R+1})^{1/d} \rightarrow 1$ as $d\rightarrow \infty$ and has to be ruled out. Consequently, we have $\max_{i\in[d-1]} |\lambda_i|$ close to $R+1 = \frac{p}{q}$ and $|\lambda_i|$ close to $(\frac{R}{R+1})^{1/d} = (\frac{p-q}{p})^{1/d}$ for the remaining cases. The numerical results in Figure~\ref{fig:EIGbeta} support our above presented argument.
\begin{figure}[!htb]
\centering
  \includegraphics[width=0.485\linewidth]{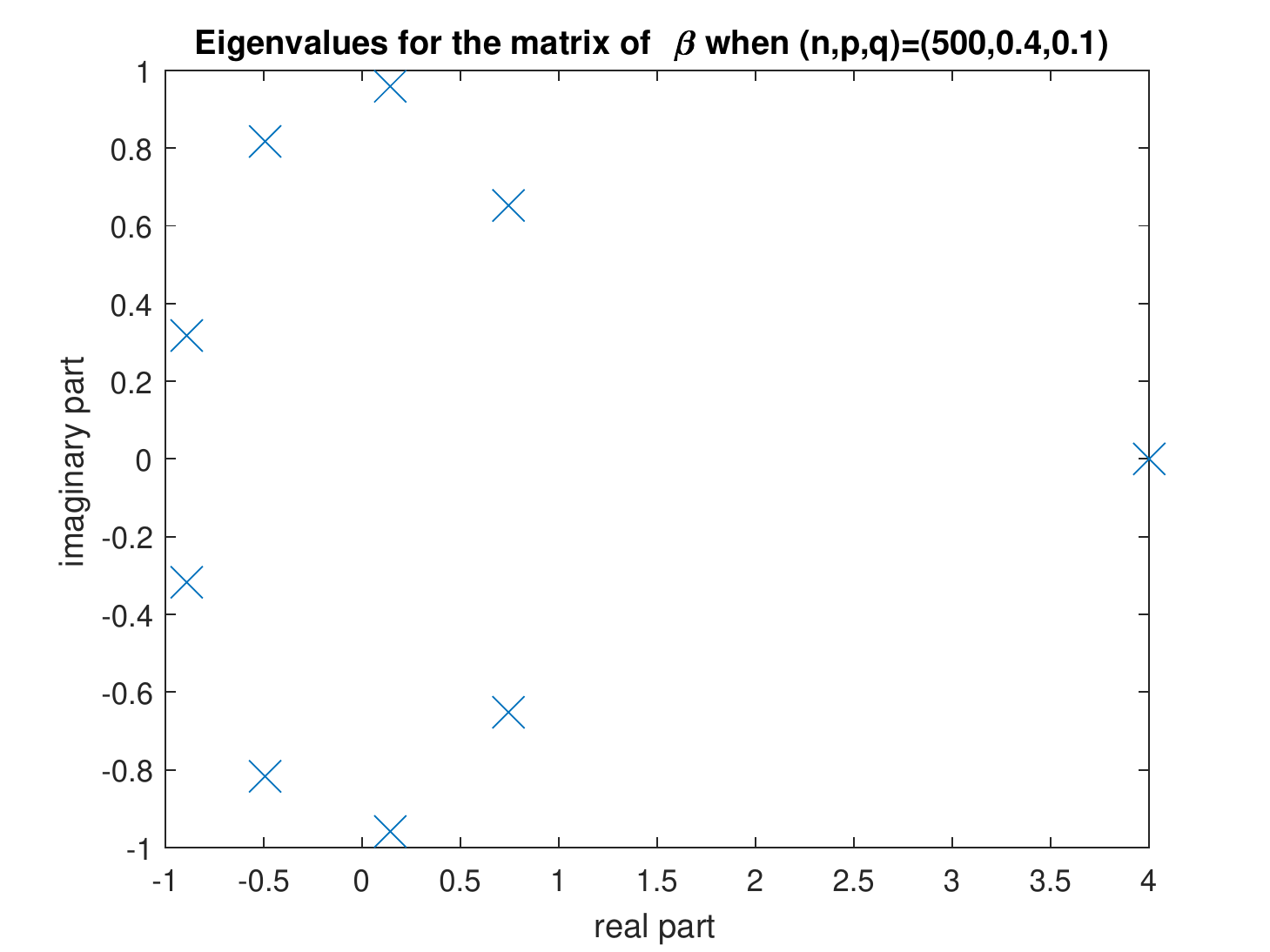}
  \includegraphics[width=0.485\linewidth]{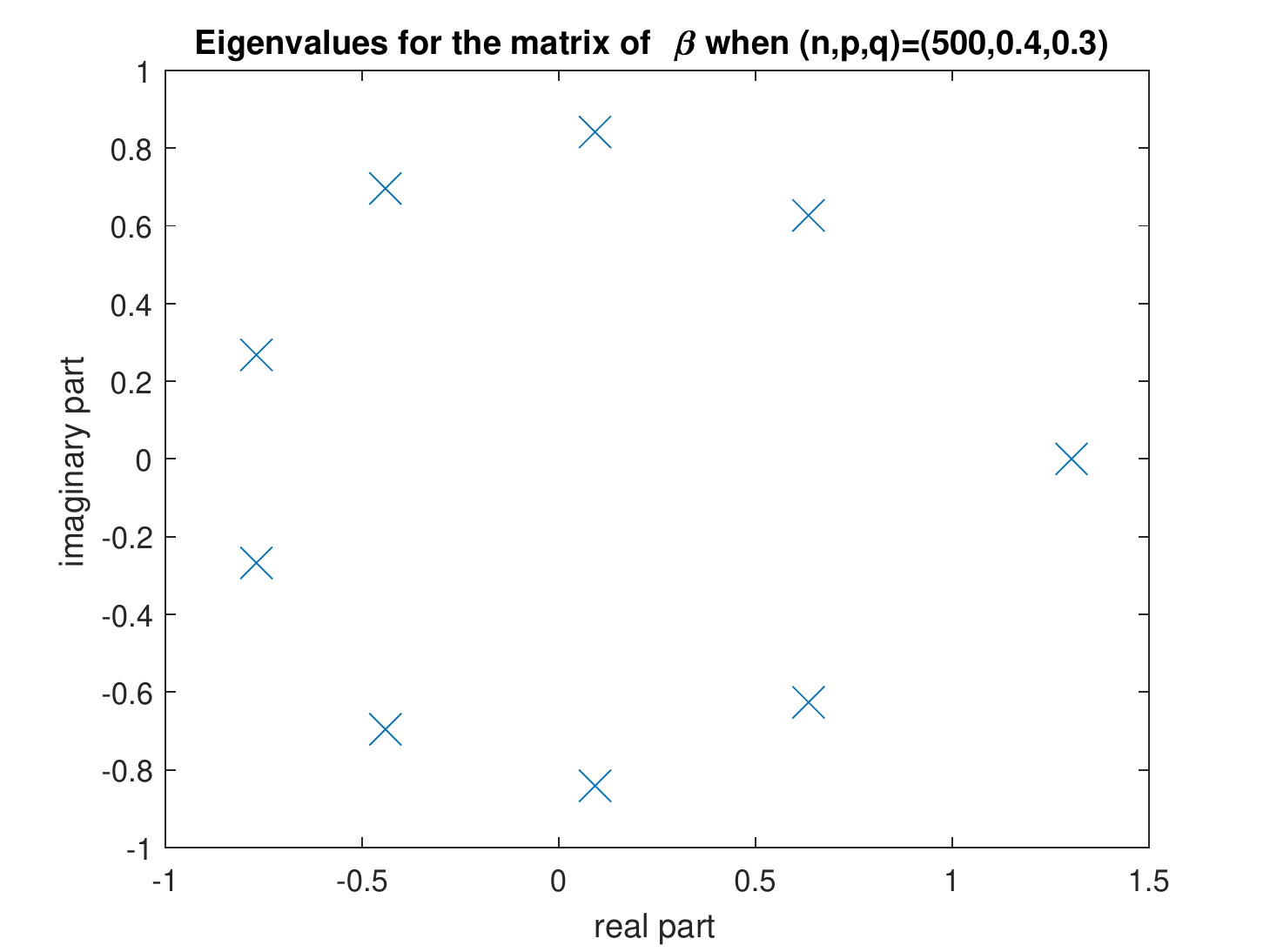}
  \caption{The roots of~\eqref{eq:eigenbeta1}.}\label{fig:EIGbeta}
\end{figure}

Next we evaluate the eigenvalues of the update matrix for $\zeta$, which is a significantly more complicated task. Again we first rewrite the update matrix in~\eqref{margsimp8} as:
\begin{equation}
    \frac{nq}{2}\times
    \begin{bmatrix}
     	2 & \cdots & 0 & 0 & R\\
        1 & 0 & \cdots & 0 & R\\
        0 & \ddots & \vdots & 0 & R\\
        0 & \cdots & 1 & 0 & R\\
        0 & \cdots & 0 & 1 & R\\
     \end{bmatrix} \triangleq \frac{nq}{2} \mathbf{Q_2}, R \triangleq \frac{p-q}{q}.
\end{equation}
Note that we cannot write down a simple explicit formula for of the eigenvalues of $\mathbf{Q_2}$ now, as the matrix is not a companion matrix. Still we can show that the characteristic polynomial of $\mathbf{Q_2}$ reads as follows
\begin{equation}\label{eq:eigenzeta1}
    t^{d-1} - (2+R)t^{d-2} + R(t^{d-3}+t^{d-4}+...+1) = 0.
\end{equation}
Similarly, $t = 0$ is not a root of the equation since $p>q \Rightarrow R>0$. Moreover, $t = 1$ is also not a root unless $R = \frac{1}{d-3}$, which we rule out for simplicity of analysis. Once again multiplying both sides of the equations by $(t-1)$ we have
\begin{align}
    & t^{d} - (3+R)t^{d-1} + 2(1+R)t^{d-2} = R \nonumber\\
    & \Leftrightarrow (t-2)(t-(1+R)) = \frac{R}{t^{d-2}}.
\end{align}
Obviously, when $|t| > 1$ there are two real eigenvalues close to $2$ and $1+R$ whenever $d$ is sufficiently large. On the other hand, the remaining eigenvalues are complex and contained within the ring $\{z\in \mathbb{C}|\;\; |(\frac{R}{2(1+R)})^{\frac{1}{d-1}}|\leq |z| \leq 1\}$. Numerical results also confirm this finding as illustrated in Figure~\ref{fig:EIGzeta}.
\begin{figure}[!htb]
\centering
  \includegraphics[width=0.485\linewidth]{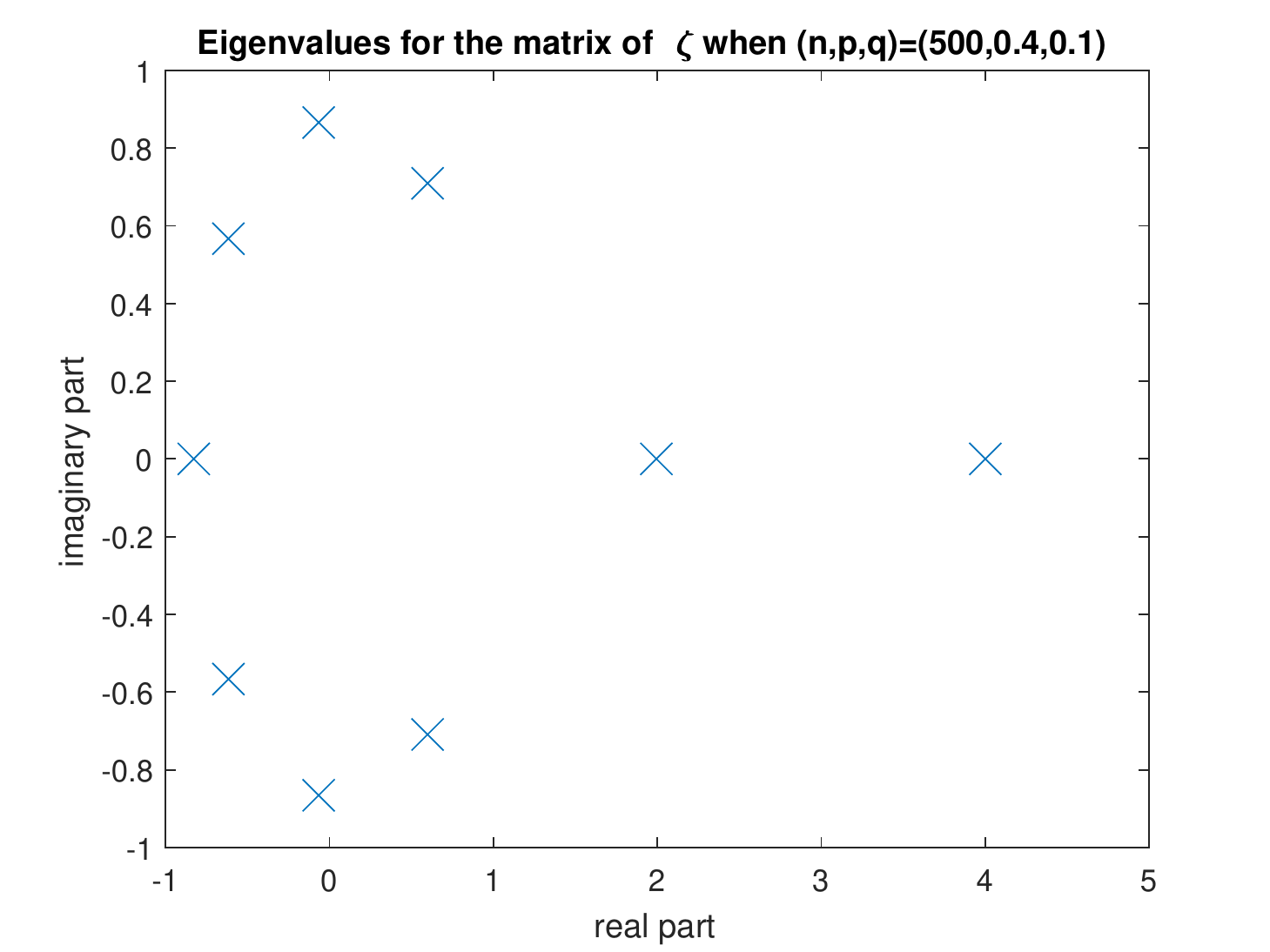}
  \includegraphics[width=0.485\linewidth]{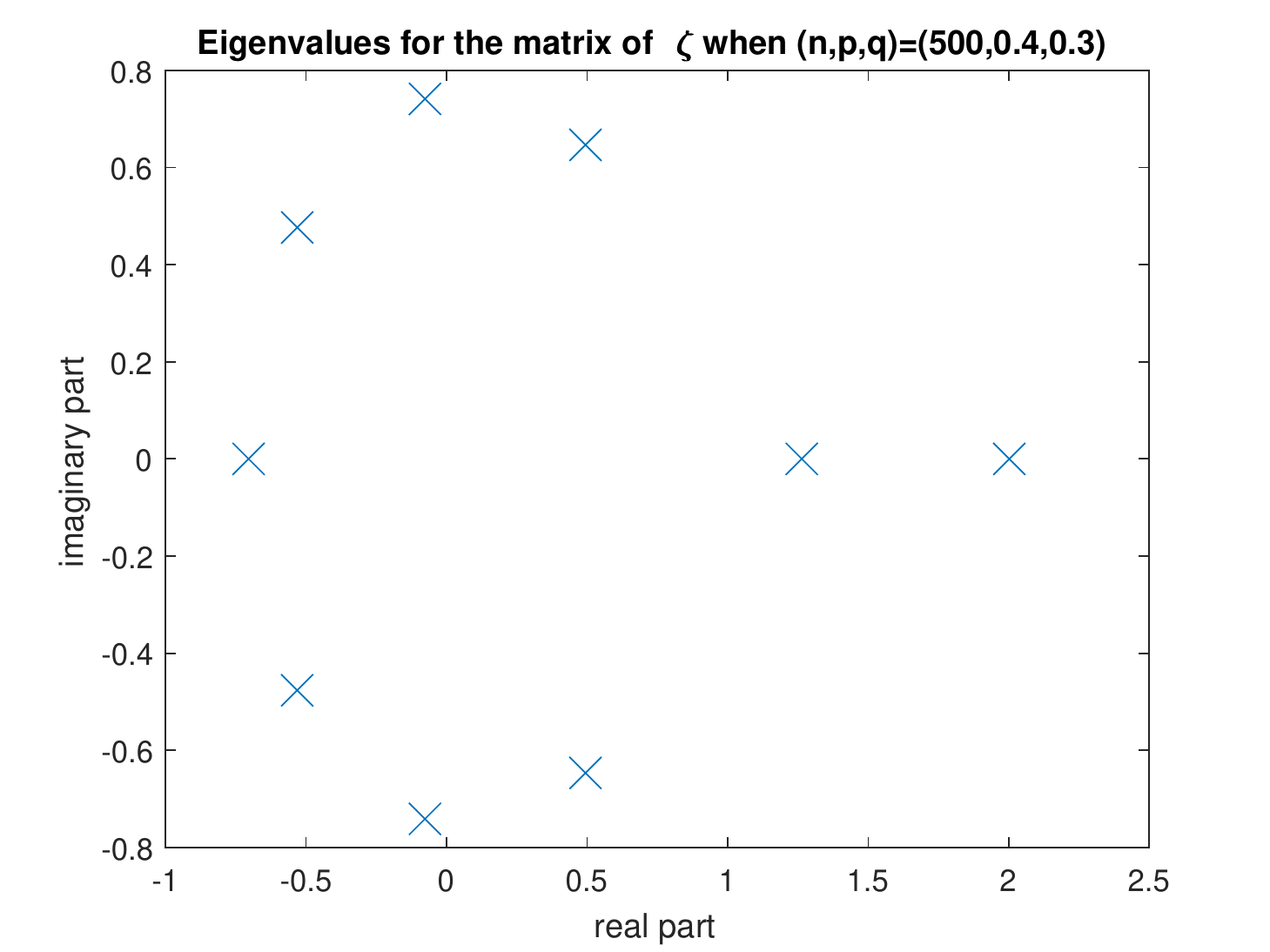}
  \caption{The roots of~\eqref{eq:eigenzeta1}.}\label{fig:EIGzeta}
\end{figure}

Thus by considering only the leading terms in $\beta,\zeta$, $\bar{w}_k$ we may write
$$\frac{2}{n}\frac{C_3(1+R)^k}{C_1(2)^k + C_2(1+R)^k} = \frac{2}{n}\frac{C_3(p)^k}{C_1(2q)^k + C_2(p)^k}.$$

\section{Closed form results for Theorem~\ref{thm:abapprox} and $d=3$}\label{app:d3case}
When $d=3$, we can have the following closed-form characterization of the centroid distance.
% We start from the case when $G$ is generated from $3$-hSBM$(n,p,q)$ as a warming-up. In contrast to CE-RWoH, generalization the results for $d=3$ to arbitrary $d$ is not trivial while we will provide detailed discussion for the large $d$ case later.  %The Later, we will generalize the results to arbitrary $d$.
%We show that when $G$ is generated from $3$-hSBM$(n,p,q)$, then we have the following result
\begin{corollary}\label{thm:d3case}
    Let $G$ be sampled from a $3$-hSBM$(n,p,q)$ and let the tensor RWoHs be associated with a initial vector $y^{(0)}_{s_1,s_2; h} = 1$ and $y^{(0)}_{v_1,v_2;h} = 0,\;\forall (v_1,v_2)\neq (s_1,s_2)$, where $s_1$ and $s_2$ are chosen independently and uniformly at random from $V_0$. Let $\bar{y}^{(0)}_{h} = \E y^{(0)}_{h}$. Then, the mean-field LPs $\bar{x}_{h}^{k}$ of the tensor RWoHs will have exactly two geometric centroids $\bar{a}^{k},\bar{b}^{k}$ that satisfy
    \begin{align*}
  \bar{a}^{(k)} - \bar{b}^{(k)} = \frac{2}{n}\frac{\beta_1(k)}{\zeta_1(k)}, %\textcolor{blue}{2/n?}
    \end{align*}
    where
    \begin{align*}
        &\zeta_1(k) = \frac{\sqrt{R^2+4}-R-2}{2\sqrt{R^2+4}}\times (\frac{nq}{4}(R-\sqrt{R^2+4}+2))^k \\
        &+ \frac{\sqrt{R^2+4}+R+2}{2\sqrt{R^2+4}}\times (\frac{nq}{4}(R+\sqrt{R^2+4}+2))^k,\\
        &\beta_1(k) = \frac{1}{2}\left(1-\sqrt{\frac{R}{R+4}}\right)\times (\frac{nq}{4}(R-\sqrt{R}\sqrt{R+4}))^k \\
        &+ \frac{1}{2}(1+\sqrt{\frac{R}{R+4}})\times(\frac{nq}{4}(R+\sqrt{R}\sqrt{R+4}))^k,
    \end{align*}
     and $R = \frac{p-q}{p}$.
\end{corollary}
Corollary~\ref{thm:d3case} directly follows from the eigenvalue decomposition of the matrices in Theorem~\ref{thm:abapprox}.
% \section{Proof of Theorem~\ref{cor:HgeqCE}}
% From Theorem~\ref{thm:abapprox}, we know that the characteristic function of $\beta$ will be
% $$t^{d-1} = R(t^{d-2}+\cdots+1)$$
% and for $\zeta$
% $$t^{d-1} = (2+R)t^{d-2}-R(t^{d-3}+\cdots+1),$$
% where $R = \frac{p-q}{q}$. Hence we know that $\beta_1,\zeta_1$ should satisfy
% \begin{align*}
%     & \beta_1(k) = R(\beta_1(k-1)+\cdots+ \beta_1(k-(d-1)))\\
%     & \zeta_1(k) = (2+R)\zeta_1(k-1)-R(\beta_1(k-2)+\cdots+ \beta_1(k-(d-1))).
% \end{align*}
% Moreover, from the initial condition in Theorem~\ref{thm:abapprox} we can easily observe that $\forall k\geq 0$, $\beta_1(k)\geq 0$ and $\zeta_1(k)\geq 0$. Hence we have
% \begin{align*}
%     & \bar{w}_k = \frac{2}{n}\frac{\beta_1(k)}{\zeta_1(k)} = \frac{2}{n}\frac{R(\beta_1(k-1)+\cdots+ \beta_1(k-(d-1)))}{(2+R)\zeta_1(k-1)-R(\beta_1(k-2)+\cdots+ \beta_1(k-(d-1)))}\\
%     & \geq \frac{2}{n}\frac{\beta_1(k-1)}{\zeta_1(k-1)}\frac{R}{R+2} = \frac{R}{R+2}\bar{w}_{k-1} = \frac{p-q}{p+q}\bar{w}_{k-1}.
% \end{align*}
% Hence we complete the proof.

\section{Proof of Lemma~\ref{lma:concentrate}}\label{app:pflma1}
The proof follows along the same lines as the proof of the concentration result in~\cite{kloumann2017block}.
Given $\epsilon>0$ and the fact that $k<K$ for some constant $K$, we choose $\gamma\geq 0$ so that $\left (\frac{1-\gamma}{1+\gamma}\right )^K \geq 1-\epsilon$, $\left (\frac{1+\gamma}{1-\gamma}\right )^K \leq 1+\epsilon$.

First we set $d_{00} = \mathbb{E}\sum_{u\in V_0}A_{vu}^{(ce)},\;v\in V_0$; $d_{01}, d_{10}$ and $d_{11}$ can be defined similarly.
It is not hard to see that under the $d$-hSBM, we have $d_{00} = d_{11} = (\frac{n}{2})^{d-1}(p+(2^{d-2}-1)\times q)$ and $d_{01} = d_{10} = (\frac{n}{2})^{d-1}(2^{d-2}q)$. Next we denote the sum of random variables as following,
\begin{equation*}
    v\in V_i,\; S_{ij} \triangleq \sum_{u\in V_j}A_{vu}^{(ce)} = \sum_{v_1,...,v_{d-2}\in V}\sum_{u\in V_j}A_{vul}.
\end{equation*}
It is clear that $\mathbb{E}S_{ij} = d_{ij} \geq \frac{n^{d-1}q}{2}$. Then we need to count the number of independent random variables in the above expression for all $i,j\in\{0,1\}$ in order to apply Hoeffding's bound. Note that the number of independent random variables in summation above is at most $$\frac{\frac{n^{d-1}}{2}}{\lfloor\frac{d-1}{2}\rfloor ! \lceil \frac{d-1}{2}\rceil !},$$ where each of them appears at most $(d-1)!$ times. Hence by Hoeffding's bound
% In the first sum there are $\frac{n^2}{8}$ independent random variables and each of them appears once, since $A_{vul} = A_{vlu}$ by the symmetry property of $3$-hSBM. In the second sum we have $\frac{n^2}{4}$ independent random variables. For $v\in V_0,$ the Hoeffding's bound implies that
\begin{align*}
    & \mathbb{P}\left \{S_{ij} \notin [(1-\gamma)d_{ij},(1+\gamma)d_{ij}]\right \} \\
    &\leq 2\exp\left(-\frac{\lfloor\frac{d-1}{2}\rfloor ! \lceil \frac{d-1}{2}\rceil !}{[(d-1)!]^2} \gamma^2q^2n^{d-1}\right),\;\forall i,j\in\{0,1\}.
\end{align*}
% Similarly, for $v\in V_0$,
% \begin{align*}
%     & \mathbb{P}\left \{\sum_{u\in V_1}A_{vu}^{(ce)} \notin [(1-\gamma)d_{01},(1+\gamma)d_{01}]\right \} \\
%     & \leq 2\exp(-\frac{2n^2\gamma^2q^2}{3}).
% \end{align*}
By invoking the union bound over $S_{ij},\forall i,j\in\{0,1\}$ we have
\begin{align*}
    & \mathbb{P}\left \{S_{ij} \in [(1-\gamma)d_{ij},(1+\gamma)d_{ij}] \right \} \\
    &\geq 1-4n\exp\left(-\frac{\lfloor\frac{d-1}{2}\rfloor ! \lceil \frac{d-1}{2}\rceil !}{[(d-1)!]^2} \gamma^2q^2n^{d-1}\right).
\end{align*}
Hence if $\frac{q^2n^{d-1}}{\log(n)}\rightarrow \infty$ then for $d$ constant and $n$ large enough the following event hold
\begin{align}\label{eq:CEconcentrate1}
    S_{ij} \in [(1-\gamma)d_{ij},(1+\gamma)d_{ij}],\;\forall i,j\in\{0,1\}
\end{align}
with probability at least
$$1-4n\exp\left(-\frac{\lfloor\frac{d-1}{2}\rfloor ! \lceil \frac{d-1}{2}\rceil !}{[(d-1)!]^2} \gamma^2q^2n^{d-1}\right) = 1-o(1).$$
% \textcolor{red}{Can we be more precise about the "high probability" value?}
In the derivations that follow we condition our probability computation given this event. We prove by induction that the following two claims are true
\begin{align*}
    & M^{(k)} \triangleq \sum_{i\in V_0}y_{i;ce}^{(k)}\in [(1-\gamma)^k\bar{M}^{(k)},(1+\gamma)^k\bar{M}^{(k)}]\\
    & N^{(k)} \triangleq \sum_{i\in V_1}y_{i;ce}^{(k)}\in [(1-\gamma)^k\bar{N}^{(k)},(1+\gamma)^k\bar{N}^{(k)}],
\end{align*}
where $\bar{M},\bar{N}$ satisfy the following recurrence relations
\scriptsize
    \begin{align}\label{eq:recurMN}
    &\begin{bmatrix}
    \bar{M}^{(k)}\\
    \bar{N}^{(k)}
    \end{bmatrix}
    = (\frac{n}{2})^{d-1}
    \begin{bmatrix}
       p+(2^{d-2}-1)q  &   2^{d-2}q\\
       2^{d-2}q  &   p+(2^{d-2}-1)q
    \end{bmatrix}
    \begin{bmatrix}
    \bar{M}^{(k-1)}\\
    \bar{N}^{(k-1)}
    \end{bmatrix}\nonumber\\
    &\begin{bmatrix}
    \bar{M}^{(0)}\\
    \bar{N}^{(0)}
    \end{bmatrix}
    =
    \begin{bmatrix}
    1\\
    0
    \end{bmatrix}.
\end{align}
\normalsize
The base case holds due to the choice of the initial conditions. For the induction step, by definition we have that $\sum_{i\in V_0}y^{(k+1)}_{i;ce}$ equals
\scriptsize
\begin{align*}
        & \sum_{i\in V_0}\sum_{j\in V}A^{(ce)}_{ji}y^{(k)}_{j;ce} = \sum_{i\in V_0}\sum_{j\in V_0}A^{(ce)}_{ji}y^{(k)}_{j;ce}+\sum_{i\in V_0}\sum_{j\in V_1}A^{(ce)}_{ji}y^{(k)}_{j;ce}\\
        & \leq (1+\gamma)(d_{00}\sum_{j\in V_0}y^{(k)}_{j;ce}+d_{10}\sum_{j\in V_1}y^{(k)}_{j;ce}) \\
        &\leq (1+\gamma)^{k+1}(d_{00}\bar{M}^{(k)}+d_{10}\bar{N}^{(k)})\\
        &\leq (1+\gamma)^{k+1}\bar{M}^{(k+1)}
    \end{align*}
\normalsize
Similar arguments may be used for the lower bound and for $N$. Hence, by the definition of $x_{ce}$ we have
\scriptsize
\begin{align*}
    &\frac{2}{n}\sum_{i\in V_0}x_{i;ce}^{(k)}\in \left[\frac{(1-\gamma)^{k}}{(1+\gamma)^{k}}\frac{2}{n}\frac{\bar{M}^{(k)}}{\bar{M}^{(k)}+\bar{N}^{(k)}},\frac{(1+\gamma)^{k}}{(1-\gamma)^{k}}\frac{2}{n}\frac{\bar{M}^{(k)}}{\bar{M}^{(k)}+\bar{N}^{(k)}}\right]\\
    &\frac{2}{n}\sum_{i\in V_1}x_{i;ce}^{(k)}\in \left[\frac{(1-\gamma)^{k}}{(1+\gamma)^{k}}\frac{2}{n}\frac{\bar{N}^{(k)}}{\bar{M}^{(k)}+\bar{N}^{(k)}},\frac{(1+\gamma)^{k}}{(1-\gamma)^{k}}\frac{2}{n}\frac{\bar{N}^{(k)}}{\bar{M}^{(k)}+\bar{N}^{(k)}}\right].
\end{align*}
\normalsize
Based on our choice of $\gamma$, and the assumption that $\forall k\leq K$ we also have
\scriptsize
\begin{align*}
    &\frac{2}{n}\sum_{i\in V_0}x_{i;ce}^{(k)}\in \left[(1-\epsilon)\frac{2}{n}\frac{\bar{M}^{(k)}}{\bar{M}^{(k)}+\bar{N}^{(k)}},(1+\epsilon)\frac{2}{n}\frac{\bar{M}^{(k)}}{\bar{M}^{(k)}+\bar{N}^{(k)}}\right]\\
    &\frac{2}{n}\sum_{i\in V_1}x_{i;ce}^{(k)}\in \left[(1-\epsilon)\frac{2}{n}\frac{\bar{N}^{(k)}}{\bar{M}^{(k)}+\bar{N}^{(k)}},(1+\epsilon)\frac{2}{n}\frac{\bar{N}^{(k)}}{\bar{M}^{(k)}+\bar{N}^{(k)}}\right].
\end{align*}
\normalsize
Hence it remains to show that
$$\bar{a}^{(k)} = \frac{2}{n}\frac{\bar{M}^{(k)}}{\bar{M}^{(k)}+\bar{N}^{(k)}}$$
and
$$\bar{b}^{(k)} = \frac{2}{n}\frac{\bar{N}^{(k)}}{\bar{M}^{(k)}+\bar{N}^{(k)}}.$$
First, observe that
\begin{align}\label{eq:finally}
    &\bar{M}^{(k)}+\bar{N}^{(k)} = [1,1][\bar{M}^{(k)},\bar{N}^{(k)}]^T \\ \notag
    &= (\frac{n}{2})^{d-1}(p+(2^{d-1}-1)q)(\bar{M}^{(k-1)}+\bar{N}^{(k-1)}).
\end{align}
Replacing~\eqref{eq:finally} into~\eqref{eq:recurMN} we obtain the claimed recurrence relation for $\bar{a}$, $\bar{b}$. Given that the initial conditions agree, the result follows.

\section{Proof of Lemma~\ref{lma:Hconcentrate}}\label{app:pflma2}
The proof is almost identical to the proof of Lemma~\ref{lma:concentrate}. We first prove the concentration of $\sum_{l}A_{v_1,...,v_{d-1},l}$. Then, the second half of the proof is the same as the second half of the proof of Lemma~\ref{lma:concentrate}, provided that one modifies the arguments according to the tensor update rule (described in Section~\ref{sec:T-RWoH}). For simplicity we therefore only focus on the first half of the proof. As before, define $d_{0...0} = \mathbb{E}\sum_{l\in V_0}A_{v_1,...,v_{d-1},l},v_1,...,v_{d-1}\in V_0$ and all the other $d$ similarly. Note that $d_{0...0} = d_{1...1} = \frac{n}{2}p$ and that the remaining $d$ take the value $\frac{n}{2}q$. By Hoeffding's bound we have for $v_1,...,v_{d-1}\in V_0$,
\begin{align*}
   & \mathbf{P}\left \{\sum_{l\in V_0}A_{v_1,...,v_{d-1}l}\notin [(1-\gamma)d_{0...0},(1+\gamma)d_{0...0}] \right \} \\
    &\leq 2\exp(-np^2\gamma^2).
\end{align*}
Similarly, for $v_1,...,v_{d-1}\in V_0$,
\begin{align*}
    &\mathbf{P}\left \{\sum_{l\in V_1}A_{v_1,...,v_{d-1},l}\notin [(1-\gamma)d_{0...01},(1+\gamma)d_{0...01}] \right \} \\
    &\leq 2\exp(-nq^2\gamma^2).
\end{align*}
From the union bound over all possible $v_1,...,v_{d-1}$ and the community of $l$, we have $\forall r\in \{0,1\}$
%and $\forall v,u\in V_i,V_j,\;\forall i,j,r\in \{0,1\}$, we have

\begin{align*}
    &\mathbf{P}\Bigg \{\sum_{l\in V_r}A_{v_1,...,v_{d-1},l}\\
    &\notin [(1-\gamma)d_{\sigma(v_1)...\sigma(v_{d-1})r},(1+\gamma)d_{\sigma(v_1)...\sigma(v_{d-1})r}] \Bigg \} \\
    &\leq 4n^{d-1}\exp(-nq^2\gamma^2),
\end{align*}

where we denote $\sigma(v_i)$ to be the community of vertex $v_i$.
Therefore, if $\frac{n q^2}{\log(n)}\rightarrow \infty$, then for $n$ sufficiently large with high probability we have $\forall v_1,...,v_{d-1},\;\forall r\in \{0,1\}$,
\scriptsize
$$\sum_{l\in V_r}A_{v_1,...,v_{d-1},l}\in [(1-\gamma)d_{\sigma(v_1)...\sigma(v_{d-1})r},(1+\gamma)d_{\sigma(v_1)...\sigma(v_{d-1})r}],$$
\normalsize
with probability at least $1-4n^{d-1}\exp(-nq^2\gamma^2) = 1-o(1)$. This completes the proof.
% \textcolor{red}{The text below seems incomplete? Also, the two proofs above are not parallel in one you choose to work with in the set, and in the other outside the set arguments. Maybe reconcile?} \textcolor{olive}{I rewrite it into the case of $d$-hSBM. It should be clear now?}

\section{Proof of Theorem~\ref{cor:HgeqCE}}\label{sec:pfthm3}
From Theorem~\ref{thm:abapprox}, one can see that the characteristic function of $\beta$ equals
$$t^{d-1} = R(t^{d-2}+\cdots+1),$$
and the characteristic function of $\zeta$ equals
$$t^{d-1} = (2+R)t^{d-2}-R(t^{d-3}+\cdots+1),$$
where $R = \frac{p-q}{q}$. Hence $\beta_1$ and $\zeta_1$ satisfy
\begin{align*}
     \beta_1(k) &= R(\beta_1(k-1)+\cdots+ \beta_1(k-(d-1)))\\
     \zeta_1(k) &= (2+R)\zeta_1(k-1)\\
     &-R(\beta_1(k-2)+\cdots+ \beta_1(k-(d-1))).
\end{align*}
Moreover, from the initial condition in Theorem~\ref{thm:abapprox} we can easily observe that $\forall k\geq 0$, $\beta_1(k)\geq 0$ and $\zeta_1(k)\geq 0$. The proof then follows from
\scriptsize
\begin{align*}
    & \bar{w}_k = \frac{2}{n}\frac{\beta_1(k)}{\zeta_1(k)} \\
    &= \frac{2}{n}\frac{R(\beta_1(k-1)+\cdots+ \beta_1(k-(d-1)))}{(2+R)\zeta_1(k-1)-R(\beta_1(k-2)+\cdots+ \beta_1(k-(d-1)))}\\
    & \geq \frac{2}{n}\frac{\beta_1(k-1)}{\zeta_1(k-1)}\frac{R}{R+2} = \frac{R}{R+2}\bar{w}_{k-1} = \frac{p-q}{p+q}\bar{w}_{k-1}.
\end{align*}
\normalsize

\section{Remaining proof of Theorem~\ref{thm:abapprox}}\label{app:remainpfthm2}
 Let $(i)^{d-1}$ be the binary representation of $i$ of length $d-1$ and let $(i)^{d-1}[j]$ be the $j^{th}$ bit from the right in this representation (i.e. $(3)^{4} = 0011$ and $(3)^{4}[3] = 0$). For  arbitrary $d,$ we have $\bar{y}_{v_1,...,v_{d-1};h} = Y_{j}$ if $v_i\in V_{(j-1)^{d-1}[i]}$. It is clear that $Y_{j_1}^{(k)}$ depends only on $Y_{j_2}^{(k-1)}$ and $Y_{j_3}^{(k-1)}$ where $(j_2-1)^{d-1}[1:d-2] = (j_2-1)^{d-1}[2:d-1] = (j_3-1)^{d-1}[2:d-1]$ (i.e. state $001$ is determined by $100$ and $000$). By the similar analysis as the $d=3$ case, the recurrence relations becomes:
 \scriptsize
\begin{align*}%\label{eq:MFA2MCposteq1}
        &\begin{bmatrix}
            Y_1^{(k+1)}\\
            \vdots\\
            Y_{2^{d-1}}^{(k+1)}
        \end{bmatrix}
        =
        \begin{bmatrix}
        \frac{np}{2} & 0 & \cdots & \frac{nq}{2} & 0 & \cdots & 0\\
        \frac{nq}{2} & 0 & \cdots & \frac{nq}{2} & 0 & \cdots & 0\\
        0 & \frac{nq}{2} & 0 & \cdots & \frac{nq}{2}& 0 & \cdots & 0 \\
        0 & \frac{nq}{2} & 0 & \cdots & \frac{nq}{2}& 0 & \cdots & 0 \\
        &&&& \vdots\\
        0 & \cdots &\frac{nq}{2} & 0 & \cdots & \frac{nq}{2} \\
        0 & \cdots &\frac{nq}{2} & 0 & \cdots & \frac{np}{2}
        \end{bmatrix}
        \begin{bmatrix}
            Y_1^{(k)}\\
            \vdots\\
            Y_{2^{d-1}}^{(k)}
        \end{bmatrix},\\
        &\begin{bmatrix}
            Y_1^{(0)}\\
            \vdots\\
            Y_{2^{d-1}}^{(0)}
        \end{bmatrix}
        = \frac{2^{d-1}}{n^{d-1}}
        \begin{bmatrix}
            1\\
            0\\
            \vdots\\
            0
        \end{bmatrix}\nonumber
    \end{align*}
\normalsize
Applying these recurrence relations with the definition of $\beta, \zeta$ completes the proof.
% \section{Additional things...}

% \begin{figure}[!htb]
% \centering
% %   \subfigure[$4$-hyperedge]{\includegraphics[width=0.32\linewidth]{4-hyperedge.PNG}}
%   \subfigure[A $4$-hyperedge clique-expansion to edges]{\includegraphics[width=0.3\linewidth]{CE.PNG}}
%   \quad\quad
%   \subfigure[A $4$-hyperedge clique-expansion to $3$-hyperedges]{\includegraphics[width=0.3\linewidth]{CET.PNG}}
%   \caption{Explanation of the idea of CET.}\label{fig:explain2}
% \end{figure}
\end{document}